\documentclass[twoside,11pt]{article}

\usepackage{jmlr2e}
\usepackage{amsmath}
\usepackage{tikz}
\usepackage{pgfplots}
\usetikzlibrary{automata, positioning}
\usepackage{mathrsfs}
\usepackage{amssymb}
\usepackage{graphicx}
\usepackage{subcaption}

\begin{document}

\title{Generalization Bounds for Convolutional Neural Networks}\footnote{This Thesis has been Submitted to the School of Computer Science, the University of Sydney
in Fulfilment of the requirements for the Degree of Master of Philosophy of Shan Lin on 26-Dec-2018, but has not been published. The thesis is available at \url{https://ses.library.usyd.edu.au/handle/2123/20315}.
}

\author{\name Shan Lin \email{slin4492@uni.sydney.edu.au} \\
       \addr Department of Computer Science\\
       University of Sydney\\
       Sydney, New South Wales, Australia
       \AND
       \name Jingwei Zhang \email zjin8228@uni.sydney.edu.au\\
        \addr Department of Computer Science\\
       University of Sydney\\
       Sydney, New South Wales, Australia}
       

\editor{}

\maketitle

\begin{abstract}
	Convolutional neural networks (CNNs) have achieved breakthrough performances in a wide range of applications including image classification, semantic segmentation, and object detection. Previous research on characterizing the generalization ability of neural networks mostly focuses on fully connected neural networks (FNNs), regarding CNNs as a special case of FNNs without taking into account the special structure of convolutional layers. In this work, we propose a tighter generalization bound for CNNs by exploiting the sparse and permutation structure of its weight matrices.
	As the generalization bound relies on the spectral norm of weight matrices, we further study spectral norms of three commonly used convolution operations including standard convolution, depthwise convolution, and pointwise convolution. Theoretical and experimental results both demonstrate that our bounds for CNNs are tighter than existing bounds.  
	
\end{abstract}


\section{Introduction}\label{sec-intro}
Convolutional neural networks (CNNs) have led state-of-the-art performances in various applications including image classification \citep{krizhevsky2012imagenet}, semantic segmentation \citep{long2015fully}, and object detection \citep{ren2015faster,redmon2016you}. They have significantly outperformed traditional machine learning approaches on large-scale datasets. However, 
there exists a gap between their experimental success and theoretical understanding, as previous tools in statistical learning theory become powerless for neural networks. A commonly held view suggests that a heavily overparameterized neural network can easily overfit \citep{zhang2016understanding} and thus will result in poor predictions for unseen data, which is clearly in contrast to empirical evidence. 


This motivates us to study neural networks from a theoretical perspective so as to better understand how deep learning algorithm works. We focus on a key problem in supervised learning, which is the models' ability to \textit{generalize}. For classification, the generalization performance of a model can be measured by the expected zero-one loss on the underlying input distribution, referred to as \textit{classification error}. 
Indeed, generalizable models are attained through the combined efforts from various aspects including networks architectures and optimization algorithms. In this work, we pay close attention to architectures and parameters, and propose a generalization bound specifically designed for CNNs. 


Previous research on characterizing the generalization capability of neural networks mostly focuses on fully connected neural networks (FNNs), regarding CNNs as a special case of FNNs. Although these bounds are also applicable for CNNs, the sparsity characteristic of weight matrices for convolutional layers is not fully exploited. Different from FNNs, CNNs have at least one convolutional layer which uses convolution operation instead of matrix multiplication. In this operation, a small-sized filter is applied to the feature map in a sliding-window fashion to obtain properties of sparse interaction and parameter sharing. 
We show that convolutional layers can be transformed to fully connected layers whose weight matrices are sparse and have shared weights. By exploiting this structure, it is expected to achieve a tighter bound for CNNs compared with those bounds derived from FNNs.



Our approach to deriving generalization bounds for CNNs is closely related to \citep{bartlett2017spectrally,neyshabur2017pac,golowich2018size,li2018tighter}. In these works, they upper bounded generalization error of a given neural network in terms of complexity measures depending on the Lipschitz constants of activation functions and norms of weight matrices. Different from previous methods, we propose to bound model complexity by the norm of convolutional weights rather than that of the corresponding fully connected weight matrices generated from convolutional weights. 
To achieve this goal, we make use of the sparsity and permutation structure of the convolutional layer, yielding a tighter bound in comparison with those bounds directly derived from FNNs. Our main contributions are summarized as follows.
\begin{itemize}
    \item We propose a generalization bound for general convolutional neural networks that contain both fully connected layers and convolutional layers. By transforming convolution operation into a multiplication of input feature map and the corresponding fully connected matrix generated by convolutional weights, we are able to compare the proposed generalization bound with existing bounds for FNNs. We show that the proposed bound is tighter under certain conditions, and further provide experimental results on architectures of MobileNet V1 \citep{howard2017mobilenets} and MobileNet V2 \citep{sandler2018mobilenetv2} to validate our theoretical results.
    \item   We analyzes spectral norms of different convolution operations including standard convolution operation, depthwise operation, and pointwise operation. We show that the spectral norm of weight matrices generated by convolutional weights can be upper bounded by certain norms of the convolutional weights. The proposed generalization bound relies on the spectral norm of weight matrices of layers in CNNs.
\end{itemize}
The rest of this paper is organized as follows. We first review related works in Section \ref{sec-related}. In Section \ref{sec-pre}, we briefly introduce notations, definitions, and lemmas that will be used later in this paper. Then, we present our main theorem on bounding generalization error of CNNs in Section \ref{sec-bound}. Section \ref{sec-mobile} provides discussions on spectral norm of different convolution operations, and Section \ref{sec-comparison} provides theoretical and experimental comparisons of various existing generalization bounds. Section \ref{sec-conclude} concludes the paper.


\section{Related Work}\label{sec-related}

Research on characterizing generalization behaviours of neural networks has attracted increasing attention over the last few years. \cite{neyshabur2017exploring} investigated multiple measures of model capacity that can be used to bound the generalization error of deep neural networks (DNNs), including the number of parameters, VC dimension, flatness and sharpness, and margin and norm based complexities. 
Among these complexity measures, the number of parameters and VC dimension increase as the size of network grows \citep{anthony2009neural}. Besides, proper definitions of flatness are currently lacked \citep{dinh2017sharp}. Thus both of them are not effective in explaining generalization of deep learning. In terms of margin and norm based complexities, let $n$ denote the size of a given sample, $L$ denote the depth of networks, $D$ denote the maximum width of each layer, and $A_{i}$ be the weight matrix for layer $i \in \{1, ..., L\}$. For a fully connected neural network with rectified linear unit (ReLU) activation function, \cite{neyshabur2015norm} introduced a complexity measure based on a broad definition of norm with an exponential dependence on depth. The bound is $\widetilde{\mathcal{O}}(2^L\prod_{i=1}^{L}\|A_i\|_F/\sqrt{n})$, where $\|\cdot\|_F$ represents the $F$-norm of a matrix.
Later, \cite{bartlett2017spectrally} and \cite{neyshabur2017pac} proposed similar bounds by the product of Lipschitz constants of activation functions and norms of weight matrices, which is 
    $\widetilde{\mathcal{O}}\left(\prod_{i=1}^L \|A_i\|_\sigma \sqrt{L^3D}/\sqrt{n}\right)$, where $\|\cdot\|_\sigma$ denotes the spectral norm of a matrix, which equals to the largest singular value.
Recently, \cite{golowich2018size} proposed a generalization bound of $\widetilde{\mathcal{O}}\left(\prod_{i=1}^L \|A_i\|_F \cdot  \min\left\{1/\sqrt[4]{n}, \sqrt{L/n}\right\} \right)$, which is fully independent of the depth of neural nets at the cost of a slower convergence rate. \cite{li2018tighter} achieved a tighter bound of $\widetilde{\mathcal{O}}\left(\prod_{i=1}^L \|A_i\|_\sigma \sqrt{LD^2}/\sqrt{n} \right)$ for ultra-deep neural networks when the layer depth $L$ is much greater than $D$. More recently, \cite{arora2018stronger} experimentally showed that the weight matrices of layers are noise tolerant in the sense that the influences will decay at higher layers. On account of this property, they proposed to compress the neural nets first and obtained a tighter generalization bound. 

Besides generalization bounds, researchers are also interested in understanding neural networks theoretically from other aspects including hardness of training neural networks \citep{blum1989training,livni2014computational,brutzkus2017globally,shamir2018distribution}, expressiveness of neural networks \citep{hornik1989multilayer,bengio2011expressive,eldan2016power,sharir2017expressive}, the surfaces of loss functions \citep{hochreiter1995simplifying,choromanska2015loss,keskar2016large,kawaguchi2016deep,dinh2017sharp,pmlr-v80-draxler18a,garipov2018loss}, stochastic gradient descent \citep{dauphin2014identifying,ge2015escaping,ge2016matrix,jin2017escape}, and information theory \citep{xu2017information,zhang2018information}. 

Focusing on CNNs, \cite{li2018tighter} discussed generalization bounds of CNNs composed of orthogonal filters and \cite{arora2018stronger} investigated model compression of CNNs. \cite{du2018many} studied the sample complexity of CNNs based on one convolutional filter, and showed that the sample complexity of CNNs is $\mathcal{O}(m/\epsilon^2)$ for a $m$-dimensional filter, which is much smaller than that of FNNs $\mathcal{O}(d/\epsilon^2)$ as $m\ll d$. We also refer readers to \citep{sharir2017expressive,brutzkus2017globally,du2017convolutional,du2018improved} for more theoretical analysis on CNNs.

\section{Preliminaries}\label{sec-pre}
In this section, we briefly introduce the definitions of generalization error, Rademacher complexity, and ramp loss for multiclass classification. We also formulate the notations of neural networks including fully connected neural networks, fully convolutional neural networks, and general convolutional neural networks.

\subsection{Generalization Error}

In a classification problem, training examples are assumed to be independently and identically distributed according to some fixed but unknown probability distribution $\mathcal{D}$. Given a hypothesis set $\mathcal{F}$, the expected risk of a hypothesis $f\in\mathcal{F}$ with respect to distribution $\mathcal{D}$ is defined by
	\begin{equation}\label{eq:expected-risk}
		\mathcal{R}_{\mathcal{D}}(f) = \mathbb{E}_{(x,y)\sim\mathcal{D}}[\mathbf{1}_{f(x)\neq y}].
	\end{equation}
Our goal is to learn a hypothesis $f\in\mathcal{F}$ which minimizes the expected risk. However, the expected risk is not directly accessible as the underlying distribution is unknown given a finite sample. Instead, we use an unbiased estimator of expected risk named empirical risk. Given a sample $S = \left((x_1, y_1), ..., (x_n, y_n)\right)$, the \textit{empirical risk} of hypothesis $f\in\mathcal{F}$ is defined by
	\begin{equation}\label{eq:expected-risk2}
		\widehat{\mathcal{R}}_{S}(f) = \frac{1}{n}\sum_{i=1}^{n}\mathbf{1}_{f(x_i)\neq y_i}.
	\end{equation}
Minimizing empirical risk will probably lead to overfitting, especially when the hypothesis set is overly expressive, which results in a large gap between empirical risk and expected risk. According to statistical learning theory, such gap can be bounded by the empirical Rademacher complexity of the hypothesis set, which is defined in the following.

\begin{definition} (Empirical Rademacher complexity) Given a sample $Z=\left(z_1,...,z_n\right)\in \mathcal{Z}^n$ and a set of real-valued functions $\mathcal{F}$ defined on $\mathcal{Z}$, the empirical Rademacher complexity is defined by
\begin{equation}
    \mathfrak{R}_Z(\mathcal{F}) = \frac{1}{n}\mathbb{E}_{\tau}\left[\sup_{f\in\mathcal{F}}\sum_{i=1}^{n}\tau_i f(z_i)\right],
\end{equation}
where $\tau_i$s are Rademacher variables which are uniform random variables taking values in $\{-1, +1\}$. 
\end{definition}
    
    
\subsection{Loss Function}


As the zero-one loss is non-smooth, it is difficult to minimize empirical risk. In practise, we use surrogate loss functions such as hinge loss and cross entropy loss. Here, we focus on ramp loss which is defined using the concept of \textit{margin}. For a classification problem with $k$ target classes, the margin of a training example $(x,y)\in\mathcal{X}\times\mathcal{Y}$ with respect to a hypothesis $f\in\mathcal{F}$ is defined by 
\begin{equation} 
    f(x)_y - \max_{j\neq y,j\in\{1,...,k\}}f(x)_j. 
\end{equation}
Then the ramp loss $\ell_{\eta}$ is given by 
 	\begin{equation}\label{def: ell eta}
 	    \ell_{\eta}\left(f(x), y\right) \triangleq g_{\eta}\left(-\left(f(x)_y - \max_{j\neq y}f(x)_j\right)\right),
 	\end{equation}
where $g_{\eta}(r):\mathbb{R}\to\mathbb{R}^{+}$ is a function defined by
	\[
		g_{\eta}(r) =
		\begin{cases}
			0, &{ r < -\eta} \\
			1 + r/\eta, &{-\eta\leq r\leq 0} \\
			1, &{r>0}.
		\end{cases} \]
Thus, the empirical risk for a classification problem with respect to ramp loss $\ell_{\eta}$ and the training sample $S$ can be formulated as
	\begin{equation}\label{eq:empirical-risk}
		\widehat{\mathcal{R}}_{S,\ell_{\eta}}(f) = \frac{1}{n}\sum_{i=1}^{n} \ell_{\eta}\left(f(x_i), y_i\right) = \frac{1}{n}\sum_{i=1}^{n}g_{\eta}\left(-\left(f(x_i)_{y_i} - \max_{j\neq y_i}f(x_i)_j\right)\right).
	\end{equation}
With this in view, the expected risk can be bounded in the following theorem.
\begin{theorem}\label{th:rademacher}{(\cite{bartlett2017spectrally}, Lemma 3.1)} Given a hypothesis set $\mathcal{F}$, a sample $S = ((x_1, y_1), ..., (x_n, y_n))$ of size $n$, and any $\eta>0$,
with probability at least $1-\delta$, each hypothesis $f\in\mathcal{F}$ satisfies
\begin{equation}
    \mathcal{R}_{\mathcal{D}}(f) \leq \widehat{\mathcal{R}}_{S,\ell_{\eta}}(f) + 2\mathfrak{R}_S(\mathcal{F}_{\eta}) + 3\sqrt{\frac{\ln(1/\delta)}{2n}},
\end{equation}
where $\mathcal{F}_{\eta}$ is the $\ell_\eta$ loss function class with respect to hypothesis set $\mathcal{F}$ and is defined as
\begin{equation}
    \mathcal{F}_{\eta} \triangleq \left\{(x,y) \to \ell_{\eta}\left(f(x),y\right): f\in\mathcal{F} \right\}.
\end{equation} 
 
\end{theorem}

\subsection{Neural Networks}\label{subsec-nn}
Given a neural network, let $L$ denote the number of layers, and $(\sigma_1, .., \sigma_L)$ denote fixed Lipschitz functions, e.g., rectified linear unit (ReLU) and max pooling function \citep{bartlett2017spectrally}. We assume that $\sigma_i$ is $\rho_i$-Lipschitz satisfying $\sigma_i(0) = 0$. In this paper, we discuss three types of neural networks including fully connected neural networks (FNNs), fully convolutional neural networks (FCNNs), and general convolutional neural networks (CNNs). They can be formulated in a unified approach. The key observation is that the convolution operation can be represented by matrix multiplication, see Section \ref{subsubsec-matrix} for details. Let $X=(x_1,...,x_n)^\top \in \mathbb{R}^{d \times n}$ denote the given input data, where $n$ is the number of the training examples. We list the formulations in the following.
\begin{itemize}
    \item {FNNs:} Let $\mathbf{A} = (A_1, ..., A_L)$ denote the weight matrices of all layers in FNNs, where $A_i \in \mathbb{R}^{d_{i} \times d_{i-1}}$ and $d_0 = d$. Then, FNNs can be formulated as
    \begin{equation}\label{def: fnn}
	    F_{\mathbf{A}}(X) \triangleq \sigma_L(A_L\sigma_{L-1}(A_{L-1} \cdots \sigma_1(A_1 X) \cdots)).
    \end{equation}
    \item{FCNNs:} Let $\mathbf{W}=(W_1,...,W_L)$ be the convolutional weights of all layers where convolutional weight $W_i \in \mathbb{R}^{c_i \times r_i}$ contains $c_i$ convolutional filters, each of which has dimension $r_i$. Let matrix $Z_i \in \mathbb{R}^{d_i \times n}$ denote the output of the $i_{th}$ layer, which is also the input of the $(i+1)_{th}$ layer. Thus we have $Z_0 = X$ and $d_0 = d$. Let $\mu_i$ be the function representing convolution operation for the $i_{th}$ layer taking convolutional weight $W_i$ and $Z_{i-1}$ as input. Then, we can formulate $Z_i$ in terms of $Z_{i-1}$ as   
    \begin{equation} 
        Z_{i} = \sigma_{i}(\mu_i(W_i, Z_{i-1})).
    \end{equation}
    Note that the convolutional weight $W_i$ can be in any shape as long as its total dimension is $r_i$. Such convolution operation can be rewritten as matrix multiplication, see Section \ref{subsubsec-matrix} for detail. Let $\gamma_i(W_i) \in \mathbb{R}^{d_i \times d_{i-1}}$ denote the fully connected matrix generated by convolutional weight $W_i$, and we have
    \begin{equation}\label{eq:gamma-def}
        \mu_i(W_i, Z_{i-1}) = \gamma_i(W_i)Z_{i-1}.
    \end{equation}
    Hence, FCNNs can be formulated as
    \begin{equation}\label{eq:fcnn-mu-def}
    \begin{aligned}
	    &F_{\mathbf{W}}(X) \triangleq 
	    &\sigma_L(\mu_L(W_L, \sigma_{L-1}(\mu_{L-1}( W_{L-1}, \cdots \sigma_1(\mu_1(W_1, X))\cdots)))),
	\end{aligned}
    \end{equation}
    or,
    \begin{equation}\label{eq: fcnn def}
	\begin{aligned}
		&F_{\mathbf{W}}(X) \triangleq  &\sigma_L(\gamma_L(W_L) \sigma_{L-1}(\gamma_{L-1}( W_{L-1}) \cdots \sigma_1(\gamma_1(W_1)X)\cdots)).
	\end{aligned}
    \end{equation}
    \item{CNNs:} The formulations of CNNs are obtained by a combination of~\eqref{def: fnn} and~\eqref{eq: fcnn def}. Let $\mathbf{C}=(C_1,...,C_L)$ be the unified fully connected matrices, then CNNs can be written as
    \begin{equation}\label{eq: cnn def}
	    F_{\mathbf{C}}(X) \triangleq \sigma_L(C_L \sigma_{L-1}(C_{L-1} \cdots \sigma_1(C_1 X)\cdots)),
    \end{equation}
    where matrix $C_i$ equals $A_i$ if the $i_{th}$ layer is fully connected or $C_i$ equals $\gamma_i(W_i)$ if the layer is convolutional and $\gamma_i(W_i)$ satisfies~\eqref{eq:gamma-def}.
\end{itemize}

\subsubsection{the fully connected matrix for one convolutional layer}\label{subsubsec-matrix}
This section illustrates how fully connected matrices are generated by convolution operations. Let convolutional weight $W = (w^1, ..., w^c)\in \mathbb{R}^{c \times r}$ denote the $c$ convolutional filters for one layer with the same dimension $r$. Let $Z_{input} = (Z^1, ..., Z^n) \in \mathbb{R}^{d_{input} \times n}$ denote the $n$ input examples with dimension $d_{input}$. Note that we only need to figure out how one convolutional filter $w^{i},i\in\{1,...,c\}$ operates on one input example $Z^{j},j\in\{1,...,n\}$, and this can be easily extended to multiple filters and multiple training examples. 

Assuming that each filter $w^i$ performs $m$ operations on one example $Z^j$, each operation selects $r$ elements out of the $d_{input}$-dimensional features of $Z^j$. For the $k_{th}$ operation, let vector $S_k\in \mathbb{R}^r$ represent the indices of the $r$ elements chosen from the $d_{input}$-dimensional features, and let $w^i_{S_k} \in \mathbb{R}^{d_{input}}$ denote the vector derived from $w^i$ whose elements are re-arranged to the places indexed by $S_k$ while filling the other places with zeros. Analogously, let $Z^j_{S_k} \in \mathbb{R}^r$ denote the vector which selects $r$ elements of $Z^j$ according to $S_k$. Then we have 
\begin{equation}\label{eq: indexing Z}
    w^i_{S_k}Z^j = w^i Z^j_{S_k}.
\end{equation}
For a better understanding, we illustrate this procedure using an example with a 2-dimensional convolutional filter $w\in \mathbb{R}^{2\times 2}$ and a 2-dimensional input example $Z \in \mathbb{R}^{3 \times 4}$ as follows. Let
\begin{equation}
    w =
    \begin{bmatrix}
        w_{1,1} & w_{1,2} \\
        w_{2,1} & w_{2,2}
    \end{bmatrix},\; \text{and}\;
    Z = 
    \begin{bmatrix}
        Z_{1,1} & Z_{1,2} & Z_{1,3} & Z_{1,4} \\
        Z_{2,1} & Z_{2,2} & Z_{2,3} & Z_{2,4} \\
        Z_{3,1} & Z_{3,2} & Z_{3,3} & Z_{3,4}
    \end{bmatrix},
\end{equation}
then a standard convolution operation $\mu(w, Z)$ with step size one performs $m=6$ operations and generates the following result
\begin{equation}\label{eq: mu w z}
\mu(w, Z) =
    \begin{bmatrix}
        \sum_{i=1}^2 \sum_{j=1}^2 w_{i,j}Z_{i,j} & \sum_{i=1}^2 \sum_{j=1}^2 w_{i,j}Z_{i,j+1} & \sum_{i=1}^2 \sum_{j=1}^2 w_{i,j}Z_{i,j+2} \\
        \sum_{i=1}^2 \sum_{j=1}^2 w_{i,j}Z_{i+1,j} & \sum_{i=1}^2 \sum_{j=1}^2 w_{i,j}Z_{i+1,j+1} & \sum_{i=1}^2 \sum_{j=1}^2 w_{i,j}Z_{i+1,j+2} 
    \end{bmatrix}.
\end{equation}
If we reshape $Z$ to one dimension and use $1$ as the starting index, then the index vectors become
\begin{equation} 
(S_1, .., S_6)^\top = 
    \begin{bmatrix}
        1 & 2 & 5 & 6 \\
        2 & 3 & 6 & 7 \\
        3 & 4 & 7 & 8 \\
        5 & 6 & 9 & 10 \\
        6 & 7 & 10 & 11 \\
        7 & 8 & 11 & 12 \\
    \end{bmatrix},
\end{equation}
and we have
\begin{equation} \label{eq:gamma w form}
\gamma(w) =
    \begin{bmatrix}
    w_{1,1} & w_{1,2} & 0 & 0 & w_{2,1} & w_{2,2} & 0 & 0 & 0 & 0 & 0 & 0\\
    0 & w_{1,1} & w_{1,2} & 0 & 0 & w_{2,1} & w_{2,2} & 0 & 0 & 0 & 0 & 0\\
    0 & 0 & w_{1,1} & w_{1,2} & 0 & 0 & w_{2,1} & w_{2,2} & 0 & 0 & 0 & 0\\
    0 & 0 & 0 & 0 & w_{1,1} & w_{1,2} & 0 & 0 & w_{2,1} & w_{2,2} & 0 & 0\\
    0 & 0 & 0 & 0 & 0 & w_{1,1} & w_{1,2} & 0 & 0 & w_{2,1} & w_{2,2} & 0\\
    0 & 0 & 0 & 0 & 0 & 0 & w_{1,1} & w_{1,2} & 0 & 0 & w_{2,1} & w_{2,2}\\
    \end{bmatrix}.
\end{equation}
After reshaping $\mu(w, Z)$ in~\eqref{eq: mu w z} to one dimension, we have 
\begin{equation}
    \mu(w, Z) = \gamma(w)Z.
\end{equation}
Using the above notations, $\gamma(W)$ can be easily formulated by
\begin{equation}\label{eq:gamma w}
    \gamma(W) = (w^1_{S_1}, \dots, w^1_{S_m}, w^2_{S_1}, \dots, w^2_{S_m},\dots,w^c_{S_1}, \dots, w^c_{S_m})^\top \in \mathbb{R}^{d_{output} \times d_{input}},
\end{equation}
where the output dimension $d_{output}$ is equal to $mc$.

At last, it is worth mentioning that the output dimension $m$ for one filter and the index vector $S_k$ varies for different convolution operations represented by $\mu$, depending on the shape of convolutional filters, the shape of input data, and step size. As a matter of fact, this notation allows the flexibility of applying different convolution operations in different layers as long as they can all be transformed to the same kind of formulation as~\eqref{eq:gamma w form}. 

\section{Generalization Bounds for Neural Networks}\label{sec-bound}
This section presents our main theorem on generalization bounds for different neural networks. 
Our work extends the generalization bound proposed by \citep{bartlett2017spectrally} which is derived in terms of Lipschitz constants of functions and norms of matrices. We first introduce a complexity measure that will be used in the main theorem.

\begin{definition}\label{def:conv-complexity}
Given a CNN $F_\mathbf{C}$ defined in~\eqref{eq: cnn def}. Let $(\sigma_1, ...,\sigma_L)$ be some fixed functions where $\sigma_i$ is $\rho_i$-Lipschitz satisfying $\sigma_i(0)=0$. Let $\mathbf{C} = (C_1, ..., C_L)$ be fully connected weight matrices, where $C_i \in \mathbb{R}^{d_{i} \times d_{i-1}}$ and $d_0=d$. Let $S_\mathbf{A}$ and $S_\mathbf{W}$ denote the index sets of fully connected layers and convolutional layers, respectively. Then if the index $i$ is in $S_\mathbf{A}$, we have $C_i=A_i$. Otherwise, we have $i \in S_\mathbf{W}$ and $C_i = \gamma_i(W_i)\in \mathbb{R}^{d_{i} \times d_{i-1}}$, where $W_i \in \mathbb{R}^{c_i\times r_i}$ and each filter in $W_i$ is of size $r_i$. Let $(a_1,...,a_L)$ and $(s_1,...s_L)$ be some real values. Assuming that $\|W_i\|_F \leq a_i$ and $\|\gamma_i(W_i)\|_\sigma \leq s_i$ for $i \in S_\mathbf{W}$, and $\|A_i\|_F \leq a_i$ and $\|A_i\|_\sigma \leq s_i$ for $i \in S_\mathbf{A}$, we define the sensitive complexity for $F_\mathbf{C}$ with respect to $\mathbf{C}$ as

    \begin{equation}\label{def:rc}
		\mathbf{\mathscr{R}}_{\mathbf{C}} = \left(2\prod_{i=1}^{L}\rho_{i}s_i\right)\left(\sum_{i \in S_\mathbf{A}}\frac{d_i^2d_{i-1}^2a_i}{s_i} + \sum_{i \in S_\mathbf{W}} \frac{c_i^2r_i^2a_i\sqrt{d_i/c_i}}{s_i} \right)L^2.
	\end{equation}
\end{definition}
\begin{remark} FNNs and FCNNs can be viewed as special cases of CNNs when $S_\mathbf{A}$ is empty or $S_\mathbf{W}$ is empty. Hence their corresponding sensitive complexity $\mathscr{R}_\mathbf{A}$ and $\mathscr{R}_\mathbf{W}$ can be defined as
\begin{equation}\label{def:ra}
		\mathscr{R}_{\mathbf{A}} = \left(2\prod_{i=1}^{L}\rho_is_i\right)\left(\sum_{i =1}^L \frac{d_i^2d_{i-1}^2a_i}{s_i} \right)L^2,
	\end{equation}
	and
	\begin{equation}\label{def:rw}
		\mathscr{R}_{\mathbf{W}} = \left(2\prod_{i=1}^{L}\rho_is_i\right)\left(\sum_{i=1}^L\frac{c_i^2r_i^2a_i\sqrt{d_i/c_i}}{s_i} \right)L^2.
	\end{equation}
\end{remark}
Based on the above complexity measures, we obtain the following generalization bound.
\begin{theorem}{(Generalization Bound for Convolutional Neural Networks)}\label{th:main} Given a training sample $S=\{(x_1,y_1),...,(x_n,y_n)\}$ of size $n$, each $(x_i,y_i)$ is i.i.d. according to some unknown distribution $\mathcal{D}$. Let $X = (x_1, ..., x_n)^{\top} \in \mathbb{R}^{d \times n}$ be the $n$ inputs. For a convolutional neural network $F_\mathbf{C}$ defined in~\eqref{eq: cnn def}, let $(\sigma_1, ...,\sigma_L)$ be some fixed functions where $\sigma_i$ is $\rho_i$-Lipschitz satisfying $\sigma_i(0)=0$. Then with probability at least $1-\delta$, we have
	\begin{equation}\label{eq: generalization bound}
		\begin{aligned}
			\mathcal{R}_{\mathcal{D}}(F_\mathbf{C}) \leq \widehat{\mathcal{R}}_{S,\ell_{\eta}}(F_\mathbf{C}) + \mathcal{O}\left( \left(\frac{\|X\|_F\mathscr{R}_\mathbf{C}}{\eta}\right)^{\frac{1}{4}} n^{-\frac{5}{8}}+\sqrt{\frac{\ln(1/\delta)}{n}}\right),
		\end{aligned}
	\end{equation}
where $\mathscr{R}_\mathbf{C}$ represents the corresponding complexity measure for $F_\mathbf{C}$ defined in \eqref{def:rc}.
\end{theorem}
\begin{remark} Generalization bounds for FNNs $F_\mathbf{A}$ and FCNNs $F_\mathbf{W}$ can be similarly obtained by substituting the complexity term $\mathscr{R}_\mathbf{C}$ in~\eqref{eq: generalization bound} with $\mathscr{R}_\mathbf{A}$ and $\mathscr{R}_\mathbf{W}$, respectively.

\end{remark}

Before diving into the proof details of the main theorem, we first draw an outline as shown in the chart below. 

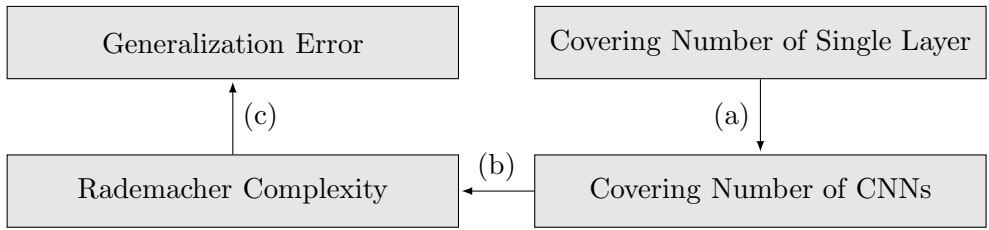
\vspace{3mm}
	    \begin{tikzpicture}
        \tikzset{node style/.style={state, 
                                    fill=gray!20!white,
                                    rectangle,
                                    minimum width=6cm}}
    
        \node[node style]               (I)   {Generalization Error};
        \node[node style, below=of I]   (II)  {Rademacher Complexity};
        \node[node style, right=of II]  (III) {Covering Number of CNNs};
        \node[node style, above=of III] (IV)  {Covering Number of Single Layer};
    
        \draw[>=latex,
          auto=right,
          every loop]
         (II)   edge node {(c)}  (I)
         (III)  edge node {(b)}  (II)
         (IV)   edge node {(a)}  (III);
        \end{tikzpicture}
        
    \vspace{2mm}
\noindent Here, 
$(a)$ denotes induction on each layer, $(b)$ represents Dudley's Entropy Integral, and $(c)$ represents Theorem \ref{th:rademacher}. 
The details are demonstrated in following subsections.
        
\begin{itemize}
    \item The first step is to bound the covering number of a single layer, see subsection \ref{subsec-single}. We consider two situations where the layers are convolutional or fully connected. 
	\item The next step is to compute the overall covering number of general convolutional neural networks by a straightforward induction on layers, as indicated by $(a)$, 
	see subsection \ref{subsec-cnn}.
	\item Finally, we relate the covering number of a convolutional neural network to its Rademacher complexity via Dudley's entropy integral as indicated by $(b)$. Substituting the above results into Theorem \ref{th:rademacher} yields our main Theorem \ref{th:main} as indicated by $(c)$, see subsection \ref{subsec-bound}.
\end{itemize}

\subsection{Covering Number Bound for a Single Neural Network Layer}\label{subsec-single}

We study covering number bounds for a single neural network layer, either fully connected or convolutional. We first introduce the definitions of $\epsilon$-cover and covering number as well as their closely related concepts named $\epsilon$-packing and packing number. 
\begin{definition}{($\epsilon$-cover and covering number)} Let $(V,\|\cdot \|)$ be a normed space and $U$ be a subset of $V$. Then $U$ is called an $\epsilon$-cover of $V$ if for any $v\in V$, there exists $u\in U$ such that
	$\|u - v\|\leq \epsilon.$ The covering number of the normed space $(V,\|\cdot\|)$ with any $\epsilon>0$ is the size of the smallest $\epsilon$-cover, which is defined by
	$\mathcal{N}(V,\epsilon,\|\cdot\|) \triangleq \min\{\,|U|: U$ is an $\epsilon$-cover of $V\}$.
\end{definition}
\begin{definition}($\epsilon$-packing and packing number) Let $(V, \|\cdot\|)$ be a normed space, and let $U\subseteq V$. Then $U$ is an $\epsilon$-packing of $V$ if for any $v,v'\in U$, the inequality $\|v-v'\|> \epsilon$ holds. The packing number is thus defined by $\mathcal{M}(V,\epsilon,\|\cdot\|) \triangleq \max\{\,|U|: U$ is an $\epsilon$-packing of $V\}$.
\end{definition}
The following lemma provides a covering number bound with respect to $l_2$ norm for a single $a$-bounded vector, which will be used in bounding the covering number of a single neural network layer. This proof extends from \citep{wu_2016}.
\begin{lemma}\label{lemma-cn}
	Let $W \triangleq \{w: w\in\mathbb{R}^r,\, \|w\|_2\leq a\}$, then for any $\epsilon>0$, the covering number of $W$ can be bounded by
	\begin{equation}\label{lemma-cn-equ}
		\ln \mathcal{N}\left(W, \epsilon, \|\cdot\|_2\right) \leq r \ln \left(1+\frac{2a}{\epsilon}\right).
	\end{equation}
\end{lemma}
\begin{proof} Let $M$ be the maximum $\epsilon$-packing of $W$ w.r.t $l_2$ norm, then $|M| = \mathcal{M}(W,\epsilon,\|\cdot\|_2)$. In the following, 
we first show that $ \mathcal{N}\left(W, \epsilon, \|\cdot\|_2\right) \leq |M|$ and then show that $|M| \leq \left(1 + \frac{2a}{\epsilon}\right)^d$.

For the first part, we demonstrate that $M$ is also an $\epsilon$-cover of $W$. Indeed, for any $w\in W, m\in M$, we must have $\|w-m\|_2 \leq \epsilon$. Otherwise, $\{w\} \bigcup M$ is a larger $\epsilon$-packing of $W$, which contradicts the definition of $M$. For the second part, let $B(w, \epsilon)$ denote the $\epsilon$-ball centered at $w$. Because $M$ is an $\epsilon$-packing of $W$, the set of $\epsilon/2$-balls centered at points in $M$ are disjoint and the union is covered within the $(a+\epsilon/2)$-ball centered at the origin. Hence,
	\begin{equation}
		\begin{aligned}
			\bigcup_{m\in M} B(m, \epsilon/2) &\leq B(0, a+\epsilon/2) \\
			|M|B(0,\epsilon/2)&\leq B(0,a+\epsilon/2) \\
			|M| (\epsilon/2)^d B(0,1) &\leq (a+\epsilon/2)^d B(0,1) \\
			|M| &\leq \left(\frac{2a}{\epsilon}+1\right)^d. 
		\end{aligned}
	\end{equation}
\end{proof}
Based on this lemma, we can bound the covering number for a single fully connected layer or a convolutional layer. In the following, we use $A$ and $W$ to denote the weight matrix for a fully connected layer and a convolutional layer, respectively. 

\begin{lemma}\label{lemma:one-fc-layer} 
Let $A \in\mathbb{R}^{d_{output}\times d_{input}}$ be the matrix for a fully connected layer satisfying $\|A\|_F\leq a$, and $Z\in\mathbb{R}^{d_{input}\times n}$ be a matrix with bounded $F$-norm representing the given input. We have the following covering number bound for this fully connected layer
    \begin{equation}
        \ln \mathcal{N}\left(\left\{AZ: \|A\|_F\leq a \right\},\epsilon,\|\cdot\|_F\right) \leq d_{input}d_{output}\ln\left(1+\frac{2a\|Z\|_F}{\epsilon}\right).
    \end{equation}
\end{lemma}
\begin{proof} 
Let $\widehat{A}$ be in the $\epsilon$-cover of $\{A: \|A\|_F\leq a\}$ such that $\|A-\widehat{A}\|_F \leq \epsilon$. Then, 
    \begin{equation}
        \|AZ - \widehat{A}Z\|_F \leq \|A - \widehat{A}\|_F \|Z\|_F \leq \|Z\|_F \epsilon.
    \end{equation}
This shows that any $\epsilon$-cover of $\{A:\|A\|_F\leq a\}$ is also an $\epsilon\|Z\|_F$-cover for $\{AZ:\|A\|_F\leq a\}$, i.e., 
    \begin{equation}\label{eq: one fc layer bound1}
        \ln \mathcal{N}\left(\left\{AZ: \|A\|_F\leq a \right\},\epsilon,\|\cdot\|_F\right) \leq \ln \mathcal{N}\left(\left\{A: \|A\|_F\leq a\right\},\frac{\epsilon}{\|Z\|_F},\|\cdot\|_F\right).
    \end{equation}
    To get the cover number of $A$, we reshape it into a one dimensional vector $\bar{A} \in \mathbb{R}^{d_{input}d_{output}}$. Then the $l_2$-norm of $\bar{A}$ is equivalent to the $F$-norm of $A$, i.e., $\|\bar{A}\|_2 = \|A\|_F \leq a$. Hence by Lemma \ref{lemma-cn}, we have
    \begin{equation}\label{eq:25}
        \ln \mathcal{N}\left(\{A: \|A\|_F\leq a\},\epsilon,\|\cdot\|_F \right) = \ln \mathcal{N}\left(\{\bar{A}: \|\bar{A}\|_2\leq a\},\epsilon,\|\cdot\|_2\right) \leq d_{input}d_{output}\ln\left(1+\frac{2a}{\epsilon}\right).
    \end{equation}
    Combining~\eqref{eq: one fc layer bound1} and~\eqref{eq:25}  concludes the proof.
\end{proof}

\begin{lemma}\label{lemma:one-conv-layer} 
Let $W=(w^{1},...,w^{c}) \in\mathbb{R}^{c\times r}$ satisfying $\|W\|_F\leq a$ and $Z\in\mathbb{R}^{d\times n}$ with bounded $\|Z\|_F$ denote the convolutional weight and the given input for one convolutional layer, respectively. Assuming that each filter $w^i$ in $W$ performs $m$ operations on $Z$, then the output is $ \mu(W, Z)\in\mathbb{R}^{mc\times n}$. We have the following covering number bound for this convolutional layer
    \begin{equation}
        \ln \mathcal{N}\left(\{\mu(W, Z): \|W\|_F\leq a\},\epsilon,\|\cdot\|_F\right) \leq cr\ln\left(1+\frac{2a\sqrt{m}\|Z\|_F}{\epsilon}\right).
    \end{equation}
\end{lemma}
\begin{proof} 
Let $\widehat{W} = (\widehat{w}^{1},...,\widehat{w}^{c})$ be in the $\epsilon$-cover of $\{W:\|W\|_F\leq a\}$ such that 
\begin{equation}\label{eq:cover pre}
    \|W-\widehat{W}\|_F = \bigg(\sum_{i=1}^c \|w^i - \widehat{w}^i\|_2^2 \bigg)^{\frac{1}{2}} \leq \epsilon.
\end{equation}
Let $\gamma(W)\in\mathbb{R}^{mc\times d}$ and $\gamma(\widehat{W})\in\mathbb{R}^{mc\times d}$ denote the corresponding fully connected weight matrices generated by convolutional weight $W$ and $\widehat{W}$. Next, we have
 \begin{equation}
        \begin{aligned}
            \|\gamma(W)Z - \gamma(\widehat{W})Z\|_F 
            \overset{(a)}{=}& \left(\sum_{i=1}^{c}\sum_{j=1}^{m}\|w^{i}_{S_j}Z- \widehat{w}^{i}_{S_j}Z\|_2^2\right)^{\frac{1}{2}} \\
            \leq & \sqrt{m} \left(\sum_{i=1}^c\max_{j}\|w^{i}_{S_j}Z- \widehat{w}^{i}_{S_j}Z\|_2^2 \right)^{\frac{1}{2}} \\
            \overset{(b)}{=} & \sqrt{m} \left(\sum_{i=1}^c\max_{j}\|(w^i - \widehat{w}^i) Z_{S_j}\|_2^2 \right)^{\frac{1}{2}} \\
            \leq & \sqrt{m} \left(\sum_{i=1}^c\max_{j}\|w^i - \widehat{w}^i\|_2^2\|Z_{S_j}\|_F^2 \right)^{\frac{1}{2}} \\
            \leq & \sqrt{m} \left(\sum_{i=1}^c\|w^i - \widehat{w}^i\|_2^2\|Z\|_F^2 \right)^{\frac{1}{2}} \\ \overset{(c)}{\leq} & \sqrt{m}\epsilon\|Z\|_F,
        \end{aligned}
    \end{equation}
    where $(a)$ is due to~\eqref{eq:gamma w}, $(b)$ is from~\eqref{eq: indexing Z}, and $(c)$ is due to~\eqref{eq:cover pre}. Hence,
    \begin{equation}
        \|\mu(W, Z) - \mu(\widehat{W}, Z)\|_F = \|\gamma(W)Z - \gamma(\widehat{W})Z\|_F \leq \sqrt{m}\epsilon\|Z\|_F.
    \end{equation}
    This shows that any $\epsilon$-cover of $\{W: \|W\|_F\leq a\}$ is also a $\sqrt{m}\epsilon\|Z\|_F$-cover of $\{\mu(W, Z):\|W\|_F\leq a \}$, i.e.,
    \begin{equation}\label{eq: one layer bound1}
        \ln \mathcal{N}\left(\{\mu(W, Z):\|W\|_F\leq a\},\epsilon,\|\cdot\|_F\right) \leq \ln \mathcal{N}\left(\{W: \|W\|_F\leq a\},\frac{\epsilon}{\sqrt{m}\|Z\|_F},\|\cdot\|_F\right).
    \end{equation}
   Similar to the proof of Lemma \ref{lemma:one-fc-layer}, we reshape $W$ into a one-dimensional vector $\bar{W} \in \mathbb{R}^{cr}$. Then the $l_2$-norm of $\bar{W}$ is equivalent to the $F$-norm of $W$, i.e., $\|\bar{W}\|_2 = \|W\|_F \leq a$. Hence by Lemma \ref{lemma-cn}, we have
    \begin{equation}\label{eq:29}
        \ln \mathcal{N}(\{W: \|W\|_F\leq a\},\epsilon,\|\cdot\|_F) = \ln \mathcal{N}(\{\bar{W}: \|\bar{W}\|_2\leq a\},\epsilon,\|\cdot\|_2) \leq cr\ln\left(1+\frac{2a}{\epsilon}\right).
    \end{equation}
    Combining~\eqref{eq: one layer bound1} and~\eqref{eq:29} concludes the proof.
\end{proof}
\begin{remark}
From Lemma \ref{lemma:one-fc-layer} and Lemma \ref{lemma:one-conv-layer}, we can understand the advantage of convolutional layers over fully connected layers as follows. Assuming that there are $c$ convolutional filters and each one generates $m$ outputs, we have $d_{output} = cm$. Let convolutional weight $W$ satisfy $\|W\|_F\leq a$, then the corresponding fully connected matrix satisfies $\|\gamma(W)\|_F\leq \sqrt{m}a$. In order to make an appropriate comparison, we assume that $\|A\|_F\leq \sqrt{m}a$. Then, given the same input $Z \in \mathbb{R}^{d_{input} \times n}$, we have
\begin{equation}
    \frac{\ln \mathcal{N}(\{\mu(W, Z): \|W\|_F\leq a\},\epsilon,\|\cdot\|_F)}{\ln \mathcal{N}(\{AZ: \|A\|_F\leq \sqrt{m}a\},\epsilon,\|\cdot\|_F)} = \mathcal{O}\left(\frac{d_{input}d_{output}}{cr}\right) = \mathcal{O}\left(\frac{md_{input}}{r}\right).
\end{equation}
Typically we have $r \ll d_{input}$, hence the convolutional layers have much tighter covering number bounds then the fully connected layers.
\end{remark}

\subsection{Covering Number Bounds for Neural Networks}\label{subsec-cnn}
In previous section, we have obtained covering number bounds for a single network layer that is either fully connected or convolutional. Based on these results, this section studies covering number bounds for multilayer neural networks including FNNs, FCNNs and CNNs. 
Our analysis depends on the following lemma which shows that the covering number of a multilayer neural network can be bounded by the product of the covering number bounds of its layers.

\begin{lemma}\label{bartlett:cn}{(\cite{bartlett2017spectrally}, Lemma A.7)} Let $(\epsilon_1,...,\epsilon_L)$ be given, along with fixed Lipschitz mappings $(\sigma_1,...,\sigma_L)$ where $\sigma_i$ is $\rho_i$-Lipschitz satisfying $\sigma_i(0)=0$. Let $\mathbf{A}=(A_1,...,A_L)$ denote the fully connected weight matrices for all layers and $X \in \mathbb{R}^{d \times n}$ be the given input data with bounded $F$-norm. Denote by $\mathcal{H}_X$ the family of matrices generated by evaluating $X$ for all neural networks $F_\mathbf{A}(X)$ defined in~\eqref{def: fnn} with bounded weights, i.e., $\mathcal{H}_X \triangleq \{F_{\mathbf{A}}(X): \|A_i\|_\sigma \leq s_i, \forall i=1, ..., L\}$. Then, letting
\begin{equation}\label{eq:epsilon}
\epsilon \triangleq \sum_{j\leq L}\epsilon_j\rho_j\prod_{l=j+1}^{L}\rho_l s_l, 
\end{equation}
we have the following $\epsilon$-covering number bound for $\mathcal{H}_X$:
\begin{equation}
    \mathcal{N}(\mathcal{H}_X,\epsilon,\|\cdot\|_F) \leq \prod_{i=1}^{L}\sup_{\mathbf{A}_{i-1},\forall j \leq i-1, \|A_j\|_\sigma \leq s_j}\mathcal{N}\left(\left\{A_iF_{\mathbf{A}_{i-1}}(X): \|A_i\|_\sigma \leq s_i\right\},\epsilon_i, \|\cdot\|_F\right),
\end{equation}
where $\mathbf{A}_{i}$ denotes $(A_1, ..., A_{i})$ and $F_{\mathbf{A}_{i}}$ denotes the network constructed using the first $i$ layers of $F_\mathbf{A}$.
\end{lemma}
Based on this lemma, we first prove the covering number bound for CNNs and then present the bounds for FNNs and FCNNs, both of which can be considered as special cases of CNNs.
\begin{lemma}\label{lemma-cnn}
Denote the input data by $X \in \mathbb{R}^{d \times n}$ with bounded $F$-norm. Let $(\sigma_1, ...,\sigma_L)$ be fixed functions where $\sigma_i$ is $\rho_i$-Lipschitz satisfying $\sigma_i(0)=0$. Let $\mathbf{C} = (C_1, ..., C_L)$ be fully connected weight matrices of all layers, where $C_i \in \mathbb{R}^{d_{i} \times d_{i-1}}$ and $d_0=d$. Let $S_\mathbf{A}$ and $S_\mathbf{W}$ denote the index set of fully connected layers and convolutional layers, respectively. Then if the layer index $i$ is in $S_\mathbf{A}$, we have $C_i = A_i$. Otherwise $i \in S_\mathbf{W}$ and $C_i = \gamma_i(W_i)\in \mathbb{R}^{d_{i} \times d_{i-1}}$, where $W_i \in \mathbb{R}^{c_i\times r_i}$ and each filter in $W_i$ is of size. Denote by $\mathcal{H}_X^\mathcal{C}$ the family of result matrices generated by evaluating $X$ for all CNNs $F_{\mathbf{C}}$ defined in~\eqref{eq: cnn def} with bounded weights, i.e.,
\begin{equation}
    \begin{aligned}
        \mathcal{H}_X^\mathcal{C} \triangleq \big\{F_{\mathbf{C}}(X): \|A_i\|_F\leq a_i, \forall i \in S_\mathbf{A};\;  \|W_i\|_F \leq a_i, \forall i \in S_\mathbf{W};\; \|C_i\|_\sigma \leq s_i,\forall i = 1, ..., L\big\}.
    \end{aligned}
\end{equation}
Let $\mathscr{R}_{\mathbf{C}}$ be defined in~\eqref{def:rc}, then we have the following covering number bound
\begin{equation}
	\begin{aligned}
		\ln \mathcal{N}\left(\mathcal{H}_X^\mathcal{C}, \epsilon,\|\cdot\|_F\right)
		\leq \left(\frac{ \|X\|_F \mathscr{R}_{\mathbf{C}}}{\epsilon}\right)^{\frac{1}{2}}.
	\end{aligned}
\end{equation}
\end{lemma}
\begin{proof}
For any $\epsilon>0$, define $(\epsilon_1,...,\epsilon_L)$ by $\epsilon_i = \frac{\alpha_i\epsilon}{\rho_i\prod_{l=i+1}^{L}\rho_ls_l}$ for any $\alpha_i > 0$ satisfying $\sum_{i=1}^L \alpha_i = 1$. Then we have
\begin{equation}\label{eq: epsilon cnn}
    \epsilon = \sum_{i\leq L}\epsilon_i\rho_i\prod_{l=i+1}^{L}\rho_is_i,
\end{equation}
which is of the form as~\eqref{eq:epsilon}. By Lemma \ref{bartlett:cn}, the covering number of $\mathcal{H}_X^\mathcal{C}$ is bounded by
\begin{equation}\label{eq:cover cnn}
    \mathcal{N}(\mathcal{H}_X^\mathcal{C},\epsilon, \|\cdot\|_F) \leq \prod_{i=1}^{L}\sup_{\mathbf{C}_{i-1}, \forall j \leq i-1, \|C_j\|_\sigma \leq s_j}\mathcal{N}\left(\left\{C_iF_{\mathbf{C}_{i-1}}(X):\|C_i\|_{\sigma}\leq s_i \right\},\epsilon_i, \|\cdot\|_F\right),
\end{equation}
where $\mathbf{C}_i = (C_1, ..., C_i)$.
\noindent By Lemma \ref{lemma:one-fc-layer} and Lemma \ref{lemma:one-conv-layer}, if $i \in S_\mathbf{A}$, we have
\begin{equation}\label{eq:ni fc}
    \ln \mathcal{N}\left(\left\{C_iF_{\mathbf{C}_{i-1}}(X):\|C_i\|_{\sigma}\leq s_i \right\},\epsilon_i, \|\cdot\|_F\right) \leq d_i d_{i-1}\ln \bigg(1 + \frac{2a_i \|F_{\mathbf{C}_{i-1}}(X)\|_F }{\epsilon_i}\bigg),
\end{equation}
and if $i \in S_\mathbf{W}$,
\begin{equation}\label{eq:ni conv}
    \ln \mathcal{N}\left(\left\{C_iF_{\mathbf{C}_{i-1}}(X):\|C_i\|_{\sigma}\leq s_i \right\},\epsilon_i, \|\cdot\|_F\right) \leq c_i r_i\ln \bigg(1 + \frac{2a_i \sqrt{d_i/c_i} \|F_{\mathbf{C}_{i-1}}(X)\|_F }{\epsilon_i}\bigg).
\end{equation}
Since $\|C_j\|_\sigma \leq s_j$, we have
\begin{equation}\label{eq:$F$-norm cnn}
	\begin{aligned}
		\|F_{\mathbf{C}_i}(X)\|_F &= \|\sigma_i(C_iF_{\mathbf{C}_{i-1}}(X)) - \sigma_i(0)\|_F\\
		&\leq \rho_{i}\|C_i\|_\sigma\|F_{\mathbf{C}_{i-1}}(X)\|_F \\
		&\leq \rho_{i}s_i\|F_{\mathbf{C}_{i-1}}(X)\|_F \\
		&\leq \|X\|_F\prod_{j=1}^{i}\rho_j s_j,
	\end{aligned}
\end{equation}
where the last inequality is attained by induction. Substituting~\eqref{eq:ni fc},~\eqref{eq:ni conv} and~\eqref{eq:$F$-norm cnn} into~\eqref{eq:cover cnn}, we have
\begin{equation}
	\begin{aligned}
		\ln \mathcal{N}\left(\mathcal{H}_X^\mathcal{C},\epsilon, \|\cdot\|_F\right) &\leq \sum_{i \in S_\mathbf{A}} \sup_{\mathbf{C}_{i-1}, \forall j \leq i-1, \|C_j\|_\sigma \leq s_j} d_id_{i-1} \ln \bigg(1 + \frac{2a_i \|F_{\mathbf{C}_{i-1}}(X)\|_F }{\epsilon_i}\bigg) \\
        &\quad +  \sum_{i \in S_\mathbf{W}} \sup_{\mathbf{C}_{i-1}, \forall j \leq i-1, \|C_j\|_\sigma \leq s_j} c_i r_i \ln \bigg(1 + \frac{2a_i \sqrt{d_i/c_i} \|F_{\mathbf{C}_{i-1}}(X)\|_F }{\epsilon_i}\bigg) \\
        &\leq \sum_{i \in S_\mathbf{A}}  d_i d_{i-1} \ln \bigg(1 + \frac{2a_i \|X\|_F\prod_{j=1}^{i-1}\rho_{j}s_j }{\epsilon_i}\bigg) \\
        &\quad + \sum_{i \in S_\mathbf{W}} c_i r_i \ln \bigg(1 + \frac{2a_i \sqrt{d_i/c_i} \|X\|_F\prod_{j=1}^{i-1}\rho_{j}s_j }{\epsilon_i}\bigg).
	\end{aligned}
\end{equation}
By the definition of $\epsilon_i$, we further have
\begin{equation}
	\begin{aligned}
		 &\ln \mathcal{N}\left(\mathcal{H}_X^\mathcal{C},\epsilon, \|\cdot\|_F\right) \\
		 &\leq \sum_{i \in S_\mathbf{A}}  d_i d_{i-1} \ln \bigg(1 + \frac{2a_i \|X\|_F\prod_{j=1}^{i-1}\rho_{j}s_j }{\frac{\alpha_i\epsilon}{\rho_i\prod_{l=i+1}^{L}\rho_ls_l}}\bigg) + \sum_{i \in S_\mathbf{W}}  c_i r_i \ln \bigg(1 + \frac{2a_i \sqrt{d_i/c_i} \|X\|_F\prod_{j=1}^{i-1}\rho_{j}s_j }{\frac{\alpha_i\epsilon}{\rho_i\prod_{l=i+1}^{L}\rho_ls_l}}\bigg) \\
		&= \sum_{i \in S_\mathbf{A}} d_i d_{i-1} \ln \bigg(1 + \frac{2a_i \|X\|_F\prod_{j=1}^{L}\rho_{j}s_j }{s_i \alpha_i \epsilon}\bigg) + \sum_{i \in S_\mathbf{W}} c_i r_i \ln \bigg(1 + \frac{2a_i \sqrt{d_i/c_i } \|X\|_F\prod_{j=1}^{L}\rho_{j}s_j }{s_i \alpha_i \epsilon}\bigg)\\
		&\leq \sum_{i\in S_\mathbf{A}} d_i d_{i-1}\bigg(\frac{2a_i \|X\|_F\prod_{j=1}^{L}\rho_{j}s_j }{s_i \alpha_i \epsilon}\bigg)^{\frac{1}{2}} + \sum_{i \in S_\mathbf{W}} c_i r_i\bigg(\frac{2a_i \sqrt{d_i/c_i} \|X\|_F\prod_{j=1}^{L}\rho_{j}s_j }{s_i \alpha_i \epsilon}\bigg)^{\frac{1}{2}}\\
		&=\bigg(\frac{2\|X\|_F\prod_{i=1}^{L}\rho_is_i}{\epsilon}\bigg)^{\frac{1}{2}} \left( \sum_{i \in S_\mathbf{A}} d_i d_{i-1}\left(\frac{a_i}{s_i\alpha_i}\right)^{\frac{1}{2}} + \sum_{i \in S_\mathbf{W}} c_i r_i\left(\frac{a_i\sqrt{d_i/c_i}}{s_i\alpha_i}\right)^{\frac{1}{2}} \right).
	\end{aligned}
\end{equation}
where the last inequality is due to the fact that $\ln(1+x) \leq \sqrt{x}$ for $\forall x\geq 0$. Define
\begin{equation}
    \Delta \triangleq \sum_{j \in S_\mathbf{A}}\frac{d_j^2d_{j-1}^2a_j}{s_j} + \sum_{j \in S_\mathbf{W}}\frac{c_j^2r_j^2a_j\sqrt{d_j/c_j}}{s_j},
\end{equation}
and let
\begin{equation}
    \alpha_i = 
    \begin{cases}
    \frac{d_i^2d_{i-1}^2a_i}{s_i \Delta }, & \forall i \in S_\mathbf{A} \\
    \frac{c_i^2r_i^2a_i\sqrt{d_i/c_i}}{s_i \Delta }, & \forall i \in S_\mathbf{W}.
    \end{cases}
\end{equation}
It is easy to see that $\sum_{i=1}^{L}\alpha_i=1$ by extracting the common factor $1/\Delta$ from all $\alpha_i$s and using the definition of $\Delta$. Then we have 
\begin{equation}
    \begin{aligned}
    \ln \mathcal{N}\left(\mathcal{H}_X^\mathcal{C},\epsilon, \|\cdot\|_F\right) &\leq \bigg(\frac{2\|X\|_F\prod_{i=1}^{L}\rho_is_i}{\epsilon}\bigg)^{\frac{1}{2}} \left(\sum_{i \in S_\mathbf{A}}\frac{d_i^2d_{i-1}^2a_i}{s_i} + \sum_{i \in S_\mathbf{W}} \frac{c_i^2r_i^2a_i\sqrt{d_i/c_i}}{s_i} \right)^{\frac{1}{2}}L \\
    &= \left(\frac{ \|X\|_F \mathscr{R}_{\mathbf{C}}}{\epsilon}\right)^{\frac{1}{2}}.
    \end{aligned}
\end{equation}
\end{proof}
By substituting all layers in CNNs to fully connected layers, i.e., $S_\mathbf{W}$ is empty, we have the following covering number bound for FNNs.
\begin{proposition}\label{lemma-fnn}
Denote the input data by $X \in \mathbb{R}^{d \times n}$ with bounded $F$-norm. Let $(\sigma_1, ...,\sigma_L)$ be fixed functions where $\sigma_i$ is $\rho_i$-Lipschitz satisfying $\sigma_i(0)=0$. Let $\mathbf{A} = (A_1, ..., A_L)$ be fully connected weight matrices of all layers, where $A_i \in \mathbb{R}^{d_{i} \times d_{i-1}}$ and $d_0=d$. Denote by $\mathcal{H}_X^\mathcal{A}$ the family of result matrices generated by evaluating $X$ for all FNNs $F_{\mathbf{A}}$ defined in~\eqref{def: fnn} with bounded weights, i.e.,
\begin{equation}
    \begin{aligned}
        &\mathcal{H}_X^\mathcal{A} \triangleq \big\{F_{\mathbf{A}}(X): \|A_i\|_F\leq a_i, \|A\|_\sigma \leq s_i, \forall i = 1,..., L\big\}.
    \end{aligned}
\end{equation}
Letting $\mathscr{R}_{\mathbf{A}}$ be defined in~\eqref{def:ra}, we have the following covering number bound
\begin{equation}
	\begin{aligned}
		\ln \mathcal{N}\left(\mathcal{H}_X^\mathcal{A}, \epsilon,\|\cdot\|_F\right)
		\leq \left(\frac{ \|X\|_F \mathscr{R}_{\mathbf{A}}}{\epsilon}\right)^{\frac{1}{2}}.
	\end{aligned}
\end{equation}
\end{proposition}
Similarly, by substituting all layers in CNNs to convolutional layers, i.e., $S_\mathbf{A}$ is empty, we have the following covering number bound for FCNNs.
\begin{proposition}\label{lemma-fcnn}
Denote the input data by $X \in \mathbb{R}^{d \times n}$ with bounded $F$-norm. Let $(\sigma_1, ...,\sigma_L)$ be fixed functions where $\sigma_i$ is $\rho_i$-Lipschitz satisfying $\sigma_i(0)=0$. Let $\mathbf{W} = (w_1, ..., w_L)$ and $(\gamma_1(W_1), ..., \gamma_L(W_L))$ be convolutional weights and their corresponding fully connected weight matrices, where $W_i \in \mathbb{R}^{c_i \times r_i}$ containing $c_i$ convolutional filters of size $r_i$. Let the output of each convolutional layer be $d_i$ dimension, then $\gamma_i(W_i) \in \mathbb{R}^{d_i \times d_{i-1}}$ and $d_0=d$. Denote by $\mathcal{H}_X^\mathcal{W}$ the family of result matrices generated by evaluating $X$ for all FCNNs $F_{\mathbf{W}}$ defined in~\eqref{eq: fcnn def} with bounded weights, i.e.,
\begin{equation}
    \begin{aligned}
        &\mathcal{H}_X^\mathcal{W} \triangleq \big\{F_{\mathbf{W}}(X): \|W_i\|_F\leq a_i, \|\gamma_i(W_i)\|_\sigma \leq s_i, \forall i = 1,..., L\big\}.
    \end{aligned}
\end{equation}
Letting $\mathscr{R}_{\mathbf{W}}$ be defined in~\eqref{def:rw}, we have the following covering number bound
\begin{equation}
	\begin{aligned}
		\ln \mathcal{N}\left(\mathcal{H}_X^\mathcal{W}, \epsilon,\|\cdot\|_F\right)
		\leq \left(\frac{ \|X\|_F \mathscr{R}_{\mathbf{W}}}{\epsilon}\right)^{\frac{1}{2}}.
	\end{aligned}
\end{equation}
\end{proposition}

\subsection{Generalization Bounds for Neural Networks}\label{subsec-bound}
By far, we have obtained the covering number bounds for multilayer neural networks. The next step is to relate the covering number to Rademacher complexity, and then we can obtain the generalization bound by Theorem~\ref{th:rademacher}. We take advantage of a standard tool in statistical learning theory named Dudley's entropy integral.
\begin{lemma}{(Dudley's entropy integral)} Let $\mathcal{F}$ be a real-valued function class taking values in $[0,1]$ and we assume that $\mathbf{0}\in\mathcal{F}$. Let $X = (x_1, ..., x_n)^\top$ be the given input containing $n$ examples and $\mathcal{F}_{ |X} = \left\{\left(f(x_1), ..., f(x_n)\right) \in [0,1]^n: f \in \mathcal{F}\right\}$, then
	\begin{equation}
		\mathfrak{R}_X(\mathcal{F}) \leq \inf_{\alpha>0}\left(\frac{4\alpha}{\sqrt{n}} +\frac{12}{n}\int_{\alpha}^{\sqrt{n}}\sqrt{\ln \mathcal{N}(\mathcal{F}_{ |X},\epsilon,\|\cdot\|_2)} \,d\epsilon\right).
	\end{equation}
\end{lemma}
\begin{lemma} \label{lemma-dudley-cnn}
Given a training sample $S = ((x_1, y_1), ..., (x_n, y_n))$ of $n$ examples, where $x_i \in \mathbb{R}^d$ and $y_i$ taking integer values from $[1,k]$. Let $X = (x_1, ..., x_n)^\top \in \mathbb{R}^{d \times n}$ be the input data. Consider the hypothesis set $\mathcal{H}$ of neural networks mapping values from $\mathbb{R}^d$ to $\mathbb{R}^k$, let $\ell_\eta$ be the ramp loss defined in~\eqref{def: ell eta}, and define the $\ell_\eta$ loss function class with respect to $\mathcal{H}$ as
\begin{equation}
    \mathcal{H}_{\eta} \triangleq \left\{(x,y) \to \ell_{\eta}\left(h(x),y\right): h\in\mathcal{H} \right\}.
\end{equation}
Then empirical Rademacher complexity of $\mathcal{H}_{\eta}$ can be bounded by
	\begin{equation}
		\mathfrak{R}_S(\mathcal{H}_{\eta}) \leq 16n^{-5/8}\left(\frac{2\|X\|_F\mathscr{R}}{\eta}\right)^{\frac{1}{4}},
	\end{equation}
where $\mathscr{R}$ denotes the corresponding sensitive complexity defined in~\eqref{def:rc},~\eqref{def:ra} and~\eqref{def:rw} for CNNs, FNNs and FCNNs, respectively.
\end{lemma}
\begin{proof}
Since $\ell_\eta$ is $\frac{2}{\eta}$-Lipschitz, we have 
\begin{equation}
        \ln\mathcal{N}\left(\mathcal{H}_{\eta | S},\epsilon,\|\cdot\|_F\right) \leq \ln\mathcal{N}\left(\mathcal{H}_{|X},\frac{\eta\epsilon}{2},\|\cdot\|_F\right)  
        \leq \left(\frac{2\|X\|_F\mathscr{R}}{\eta\epsilon}\right)^{\frac{1}{2}},
\end{equation}
where the last inequality follows from Lemma \ref{lemma-cnn}, Proposition \ref{lemma-fnn} and Proposition \ref{lemma-fcnn}. Moreover, because $\ell_{\eta}$ takes value from $[0,1]$, we can apply Dudley's entropy integral, which yields
\begin{equation}\label{eq:dudley}
	\begin{aligned}
		\mathfrak{R}_S(\mathcal{H}_{\eta}) &\leq \inf_{\alpha>0}\left(\frac{4\alpha}{\sqrt{n}} +\frac{12}{n}\int_{\alpha}^{\sqrt{n}}\sqrt{\ln \mathcal{N}(\mathcal{H}_{\eta|S},\epsilon,\|\cdot\|_F)}\,\text{d}\epsilon\right) \\
		&\leq \inf_{\alpha>0}\left(\frac{4\alpha}{\sqrt{n}} +\frac{12}{n}\int_{\alpha}^{\sqrt{n}}\left(\frac{2\|X\|_F\mathscr{R}}{\eta\,\epsilon}\right)^{\frac{1}{4}}\,\text{d}\epsilon\right) \\
		&= \inf_{\alpha>0}\left(\frac{4\alpha}{\sqrt{n}} +\frac{16}{n}\left(\frac{2\|X\|_F\mathscr{R}}{\eta}\right)^{\frac{1}{4}}\left({n}^{3/8} - \alpha^{3/4}\right)\right). \\
	\end{aligned}
\end{equation}
When the first derivative equals zero, the right hand side function achieves the minimum at $\alpha = \frac{81\|X\|_F\mathscr{R}}{128\eta n^2}$. Then, we have
\begin{equation}
    \begin{aligned}
    	\mathfrak{R}_S(\mathcal{H}_{\eta}) &\leq \left(\frac{4}{\sqrt{n}}\frac{81\|X\|_F\mathscr{R}}{128\eta n^2} + \frac{16}{n}\left(\frac{2\|X\|_F\mathscr{R}}{\eta}\right)^{\frac{1}{4}}\left({n}^{3/8} - \left(\frac{81\|X\|_F\mathscr{R} }{128\eta n^2}\right)^{\frac{3}{4}}\right)  \right) \\
    	&= 16n^{-5/8}\left(\frac{2\|X\|_F\mathscr{R}}{\eta}\right)^{\frac{1}{4}} - \frac{351}{32}\frac{\|X\|_F\mathscr{R}}{\eta n^2\sqrt{n}} \\
    	&\leq  16n^{-5/8}\left(\frac{2\|X\|_F\mathscr{R}}{\eta}\right)^{\frac{1}{4}}.
    \end{aligned}
\end{equation}
\end{proof}
Combining the above with Theorem \ref{th:rademacher} yields the main Theorem \ref{th:main}. And it implies that for any neural network $h \in \mathcal{H}$,
\begin{equation}\label{eq: final bound}
    \begin{aligned}
        \mathcal{R}_{\mathcal{D}}(h) &\leq \widehat{\mathcal{R}}_{S,\ell_{\eta}}(h) + 2\mathfrak{R}_S(\mathcal{H}_{\eta}) + 3\sqrt{\frac{\ln(1/\delta)}{2n}} \\
        &\leq \widehat{\mathcal{R}}_{S,\ell_{\eta}}(h) + 32\left(\frac{2\|X\|_F\mathscr{R}}{\eta}\right)^{\frac{1}{4}}n^{-5/8} + 3\sqrt{\frac{\ln(1/\delta)}{2n}}.
    \end{aligned}
\end{equation}

\begin{remark} Here we emphasize the dependence of our bound on two parameters, including the number of samples $n$ and the depth of networks $L$. For fully convolutional neural networks with ReLU activation, our bound is of $\mathcal{O}(L^{3/4}/\sqrt{n})$, because $\|X\|_F$ is of $\sqrt{n}$ and $\mathscr{R}$ is of $L^3$. In this way, our bound is better than $\mathcal{O}(L^{3/2}/\sqrt{n})$ in \citep{bartlett2017spectrally, neyshabur2017pac}, but it seems to be worse than $\mathcal{O}(L^{1/2}/\sqrt{n})$ in \citep{golowich2018size,li2018tighter}. However, when combined with other parameters, our bound is advantageous under mild conditions, see Section \ref{sec-comparison}.
\end{remark}

\section{Extensions to Different Convolution Operations} \label{sec-mobile}

 The proposed generalization bound depends on the $F$-norm of convolutional weight $W$ as well as the spectral norm of the corresponding fully connected matrix $\gamma(W)$. Although introducing the notion of $\gamma(W)$ offers much convenience when deriving covering number bounds for CNNs and making comparisons with existing generalization bounds as shown in Section \ref{subsec-cnn} and Section \ref{sec-comparison}, it is cumbersome to transform $W$ to $\gamma(W)$. Moreover, $\gamma(W)$ will change according to the input dimension even if $W$ remains the same. This motivates us to relate the spectral norm of $\gamma(W)$ to the norm of $W$ so as to remove the dependence on $\gamma(W)$ from the proposed generalization bounds for CNNs. In this section, we will consider different types of convolution operations including standard convolution, depthwise convolution, and pointwise convolution. 
Note that depthwise and pointwise convolutions are often chained together to form depthwise separable convolutions, which are commonly used in CNN architectures like \citep{howard2017mobilenets, sandler2018mobilenetv2, chollet2017xception}.



\subsection{Standard Convolution}
Previously, \cite{li2018tighter} demonstrated that the spectral norm of $\gamma(W)$ equals $r / l$ for orthogonal filters, where $r$ is the length of each filter and $l$ represents the stride size. We discuss the general case of standard convolution without the orthogonality condition.
\begin{proposition}\label{prop: gamma w sigma standard}
    Let $W = (w^1, ..., w^c) \in \mathbb{R}^{c \times r}$ be a convolutional weight matrix containing $c$ convolutional filters. Assume that each filter $w^i$ performs $m$ operations on input data $Z \in \mathbb{R}^{d \times n}$. For the corresponding fully connected matrix $\gamma(W) \in \mathbb{R}^{cm \times d}$, we have
    \begin{equation}
        \|\gamma(W)\|_\sigma \leq \sqrt{m}\|W\|_F.
    \end{equation}
\end{proposition}
\begin{proof}
    We have
    \begin{equation}\label{eq: gamma w sigma standard}
        \begin{aligned}
        \|\gamma(W)\|_\sigma &\leq \|\gamma(W)\|_F = \left(\sum_{i=1}^{cm}\sum_{j=1}^d \gamma(W)_{ij}^2\right)^{\frac{1}{2}}   \\
        &= \left(m\sum_{i=1}^{c}\sum_{j=1}^r W_{ij}^2\right)^{\frac{1}{2}} = \sqrt{m}\|W\|_F.
        \end{aligned}
    \end{equation}
    This completes the proof.
\end{proof}
\begin{remark} Substituting $s_i = \sqrt{m_i}a_i = \sqrt{d_i/c_i}a_i$ into~\eqref{def:rw} yields 
\begin{equation}
\begin{aligned}
    \mathscr{R}_{\mathbf{W}}  
    &= \left(2\prod_{i=1}^{L}\rho_i\sqrt{d_i/c_i}a_i\right)\left(\sum_{i=1}^Lc_i^2r_i^2 \right)L^2.
\end{aligned}
\end{equation}
Thus we can remove the dependence on $\gamma(W)$ from the generalization bound of FCNNs. The complexity $\mathscr{R}_{\mathbf{C}}$ of CNNs can also be updated analogously. However, we have made an aggressive step in the first inequality of~\eqref{eq: gamma w sigma standard} that simply bounds the spectral norm by $F$-norm, which is generally not tight and will result in loose bounds.
\end{remark}
In the following, we will derive tighter relations between $\|\gamma(W)\|_\sigma$ and the norm of $W$ for special cases of depthwise and pointwise convolutions.

\subsection{Depthwise Convolution}

Depthwise convolution is widely used because of its efficiency in the sense that each convolutional filter operates independently on different input channels rather than across channels. 
In previous discussions, we do not consider the specific shape of input features and convolutional filters by using them as one-dimension vectors of size $d$ and $r$. We refine our notations by assuming that $d$ and $r$ both have two dimensions, i.e., the spatial dimension and the channel dimension. 
For the $i_{th}$ layer, let $d_i = m_i \times c_i$ and $r_i = k_i \times c_{i-1}$, where $m_i$ denotes the spatial size of the output, $c_i$ denotes the number of channels for the output, and $k_i$ denotes the spatial size of each convolutional filter. Hence, we have the output $Z_i \in \mathbb{R}^{m_i \times c_i \times n}$ and the convolutional weight $W_i \in \mathbb{R}^{c_i \times k_i \times c_{i-1}}$. 

We consider one depthwise convolutional layer with input data $Z\in\mathbb{R}^{m\times c\times n}$. Let $W = (w^1, ..., w^c) \in \mathbb{R}^{c \times k}$ be the depthwise convolutional weight matrix containing $c$ depthwise convolutional filters, and each filter $w^i$ only has a spatial dimension $k$. Here we use the fact that in depthwise convolution, the number of filters is always equal to the number of input channels and each filter operates on one channel, i.e., no channel dimension is needed for each filter. Assuming that each $w^i$ performs $m'$ operations on one channel of $Z$, then the output is $Z' \in \mathbb{R}^{m' \times c \times n}$. The corresponding fully connected weight matrix $\gamma(W) \in \mathbb{R}^{m'c \times mc}$ can be written as
\begin{equation}\label{eq: depthwise gamma}
    \gamma(W) =  \begin{pmatrix}
        \Omega(w^1) & 0 & 0 & \cdots & 0 \\
        0 & \Omega(w^2) & 0 & \cdots & 0 \\
        \vdots & \vdots & \vdots & \vdots & \vdots \\
        0 & 0 & 0 & \cdots & \Omega(w^c)
\end{pmatrix} \text{, where }
    \Omega(w^i) \triangleq \begin{pmatrix}
        w^i_{S_1} \\
        w^i_{S_2} \\
        \vdots \\
        w^i_{S_{m'}}
\end{pmatrix} \in \mathbb{R}^{m' \times m},
\end{equation}
and the index set $S_i$s are defined in Section \ref{subsubsec-matrix}. 

We aim to calculate the spectral norm of $\gamma(W)$, which is equal to the square root of the largest eigenvalue of square matrix $\gamma(W)\gamma(W)^\top$. Given~\eqref{eq: depthwise gamma}, we have
\begin{equation} \label{depthwise gamma w square}
    \gamma(W)\gamma(W)^{\top} = \begin{pmatrix}
        \Omega(w^1)\Omega(w^1)^\top & 0 & 0 & \cdots & 0 \\
        0 &  \Omega(w^2)\Omega(w^2)^\top & 0 & \cdots & 0 \\
        \vdots & \vdots & \vdots & \vdots & \vdots \\
        0 & 0 & 0 & \cdots &  \Omega(w^c)\Omega(w^c)^\top
    \end{pmatrix},
\end{equation}
where
\begin{equation}\label{eq: omega square}
    \Omega(w^i)\Omega(w^i)^\top = \begin{pmatrix}
    w^i_{S_1}w^{i\top}_{S_1} & w^{i}_{S_1}w^{i\top}_{S_2} & \cdots & w^i_{S_1}w^{i\top}_{S_{m'}} \\
    w^i_{S_2}w^{i\top}_{S_1} & w^i_{S_2}w^{i\top}_{S_2} & \cdots & w^i_{S_2}w^{i\top}_{S_{m'}} \\
    \vdots & \vdots & \vdots & \vdots \\
    w^i_{S_{m'}}w^{i\top}_{S_1} & w^i_{S_{m'}}w^{i\top}_{S_2} & \cdots & w^i_{S_{m'}}w^{i\top}_{S_{m'}}
    \end{pmatrix}.
\end{equation}
Then it only remains to find the largest eigenvalue of $\Omega(w^i)\Omega(w^i)^\top$. Without loss of generality, we can assume that indices in $S_j$ are in ascending order. Let $l$ be the stride size and $k$ be the filter size, then we have 
\begin{equation}\label{eq: s definition}
    S_j = \left(l(j-1)+1, l(j-1)+2, \cdots, l(j-1) + k\right), \forall j \in 1, ..., m'.
\end{equation}
With this notation, we further divide depthwise convolution into two scenarios when $S_1, ..., S_{m'}$ are overlapping or non-overlapping, or equivalently, when the stride size is smaller than or equal to the filter size. Both cases have been shown to play important roles in modern CNNs. We start with the easier non-overlapping scenario.

\subsubsection{Non-overlapping Convolutional Filters}
Non-overlapping convolution has attracted increasing attention recently. From theoretical perspective, non-overlapping filters lead to neat formulations and deliver concrete analysis. \cite{brutzkus2017globally} showed that a non-overlapping convolutional neural network with Gaussian inputs can converge to global optimum in polynomial time. Although limited, non-overlapping convolutions do play an important role in recent works. The following lemma demonstrates the relation between $\|\gamma(W)\|_\sigma$ and $\|W\|_F$ for non-overlapping convolutions.


\begin{proposition}\label{prop: non depthwise} 
Let $W = (w^1, ..., w^c)\in\mathbb{R}^{c \times k}$ be a depthwise convolutional weight matrix, and let $\gamma(W)\in\mathbb{R}^{m'c \times mc}$, $\Omega(w^i) \in \mathbb{R}^{m' \times m}$, and $S_1, ..., S_{m'}$ be defined in~\eqref{eq: depthwise gamma}. If $S_1, ..., S_{m'}$ are non-overlapping, then we have
    \begin{equation}
        \|\gamma(W)\|_{\sigma} \leq \|W\|_F.
    \end{equation}
\end{proposition}
\begin{proof} 
When $S_1, ..., S_{m'}$ are non-overlapping, we have
\begin{equation}
    w^i_{S_p}w^{i\top}_{S_q} = \begin{cases}
    \left\|w^i\right\|_2^2, & \text{if }\, p=q \\
    0, & \text{otherwise.}
    \end{cases}
\end{equation}
By~\eqref{eq: omega square}, we have
\begin{equation}
    \Omega(w^i)\Omega(w^i)^\top = \text{diag}\left(\left\|w^i_{S_1}\right\|_2^2, ..., \big\|w^i_{S_{m'}}\big\|_2^2\right) = \text{diag}\left(\big\|w^i\big\|_2^2, ..., \left\|w^i\right\|_2^2\right) = \|w^i\|_2^2\; I_{m'},
\end{equation}
where $I_{m'}$ is the identity matrix with size $m'$. Combining with~\eqref{depthwise gamma w square}, we have
    \begin{equation}
        \begin{aligned}
            \|\gamma(W)\|_{\sigma} &= \sqrt{\lambda_{\max}\left(\gamma(W)\gamma(W)^{\top}\right)} = \sqrt{\max_{i=1,...,c} \lambda_{\max}\left(\Omega(w^i)\Omega(w^i)^\top\right)} \\
            &=\sqrt{\max_{i=1,...,c}\left\|w^i\right\|_2^2} \leq \sqrt{\sum_{i=1}^c\left\|w^i\right\|_2^2} = \|W\|_F.
        \end{aligned}
    \end{equation}
This completes the proof.
\end{proof}


\subsubsection{Overlapping Convolutional Filters}
Overlapping convolution operation significantly increases the expressive power of neural networks compared with non-overlapping operation \citep{sharir2017expressive}. In this scenario, we show that $\Omega(w^i)\Omega(w^i)^\top$ is a symmetric banded Toeplitz matrix \citep{gray2006toeplitz}. We first introduce its definition.

\begin{definition}[Banded Toeplitz Matrix]
Given an infinite sequence $\{t_q\}$ and a positive number $b$ for which $t_q = 0$ if $|q| > b$, the banded Toeplitz matrix with respect to this sequence is defined by
\begin{equation}
T_n \triangleq  \begin{pmatrix}
t_0 & t_{-1} & \cdots & t_{-b} & & & & & & &\\
t_1 & t_0 & & & & & & & & & \\
\vdots & & & & & & & & 0 & & \\
& & \ddots & & & & \ddots & & & & \\
t_b & & & & & & & & & & \\
& \ddots & & & & & & & & & \\
& & t_b & \cdots & t_1 & t_0 & t_{-1} & \cdots & t_{-b} & & \\
& &  &  &  &  &  &  &  & \ddots & \\
& &  &  &  \ddots &  &  & \ddots &  &  & t_{-b} \\
& & 0 &  &  &  &  &  &  &  & \vdots \\
&  &  &  &  &  &  &  &  & t_0 & t_{-1}\\
&  &  &  &  &  &  & t_b  & \cdots & t_1 & t_{0}
\end{pmatrix} \in \mathbb{R}^{n \times n}.
\end{equation}
If we further have $t_q = t_{-q}$ for all $|q|\leq b$, then $T_n$ is a symmetric banded Toeplitz matrix.
\end{definition}
For such a matrix, we have the following lemma for its eigenvalues.
\begin{lemma}\label{lemma:toeplitz} (\cite{gray2006toeplitz}, Lemma 4.1) Given a real symmetric banded Toeplitz matrix $T_n$ generated by $\{t_q\}$ with band $b$, let $\lambda_{k},k\in\{1,...,n\}$ be its eigenvalues. Then if  $\sum_{i=1}^{b}|t_i|<\infty$, for any $k$, we have
\begin{equation}
    \lambda_k \leq \sum_{i=-b}^{b} \left|t_{ |i| }\right|. 
\end{equation}

\end{lemma}
Based on this, we obtain the following bound for $\|\gamma(W)\|_\sigma$ in terms of $\|W\|_\infty$.
\begin{proposition}\label{prop: over depthwise} 
Let $W = (w^1, ..., w^c)\in\mathbb{R}^{c \times k}$ be a depthwise convolutional weight matrix, and let $\gamma(W)\in\mathbb{R}^{m'c \times mc}$, $\Omega(w^i) \in \mathbb{R}^{m' \times m}$, and $S_1, ..., S_{m'}$ be defined as~\eqref{eq: depthwise gamma}. If $S_1, ..., S_{m'}$ are overlapping, then we have
    \begin{equation}
        \|\gamma(W)\|_{\sigma} \leq \|W\|_\infty.
    \end{equation}
\end{proposition} 
\begin{proof} First, we show that $\Omega(w^i)\Omega(w^i)^\top$ is a symmetric banded Toeplitz matrix for all $i=1,...,c$, and then we can apply Lemma \ref{lemma:toeplitz}. Denote the stride size by $l$. In the overlapping scenario, the stride size $l$ is smaller than the filter size $k$. Using the definition of $S_j$ in~\eqref{eq: s definition}, if $p < q$, we have
\begin{equation}
    \left[\Omega(w^i)\Omega(w^i)^\top\right]_{p,q} = w^{\,i}_{S_p} w_{S_q}^{\,i\top} = \sum_{j=1}^{k-(q-p)l} w^{\,i}_{(q-p)l+j}w^{\,i}_j.
\end{equation}
By defining
\begin{equation}
    t^{\,i}_{s} \triangleq \sum_{j=1}^{k-s\,l} w^{\,i}_{s\,l+j}\,w^{\,i}_{j},
\end{equation}
we have $\left[\Omega(w^i)\Omega(w^i)^\top\right]_{p,q} = t^{\,i}_{ q-p }$. Since $\Omega(w^i)\Omega(w^i)^\top$ is symmetric, we have that
\begin{equation}
    \left[\Omega(w^i)\Omega(w^i)^\top\right]_{p,q} = t^{\,i}_{|p-q| }.
\end{equation}
By definition, it is a symmetric banded Toeplitz matrix with band $b = \lceil k/l \rceil$. Applying Lemma \ref{lemma:toeplitz}, we have
\begin{equation}
    \lambda_{\max}\left(\Omega(w^i)\Omega(w^i)^\top\right) \leq \sum_{q=-b}^b \sum_{j=1}^{k- |q| l} \left|w^i_{ |q| l+j}w^i_j\right| \overset{(a)}{\leq} \sum_{p=1}^k\sum_{q=1}^k \left|w^i_p w^i_q\right| = \left\|w^i\right\|_1^2,
\end{equation}
where equality of $(a)$ holds if the stride size $l$ is equal to one. Hence by~\eqref{depthwise gamma w square}, we have
    \begin{equation}
        \begin{aligned}
            \|\gamma(W)\|_{\sigma} &= \sqrt{\lambda_{\max}\left(\gamma(W)\gamma(W)^{\top}\right)} = \sqrt{\max_{i=1,...,c} \lambda_{\max}\left(\Omega(w^i)\Omega(w^i)^\top\right)} \\
            &\leq \sqrt{\max_{i=1,...,c}\left\|w^i\right\|_1^2} = \max_{i=1,...,c} \left\|w^i\right\|_1 = \left\|W\right\|_\infty.
        \end{aligned}
    \end{equation}
This concludes the proof.
\end{proof}

\subsection{Pointwise Convolution}
Different from depthwise convolutions, pointwise convolutions have channel dimensions, whereas their spatial dimensions are always equal to one. They are often used to combine the outputs of depthwise convolutions so as to form depthwise separable convolutions. In addition, they can be applied individually as bottleneck layers by setting the number of output channels to be smaller than the number of input channels, or as logit layers by setting the number of output channels to be equal to the number of output classes.

Given input data $Z \in \mathbb{R}^{m \times c \times n}$, let $W = (w^1, ..., w^{c'}) \in \mathbb{R}^{c' \times c}$ be a pointwise convolutional weight matrix containing $c'$ convolutional filters and the output data $Z' \in \mathbb{R}^{m \times c' \times n}$. Then the corresponding fully connected matrix $\gamma(W)$ can be formulated as 
\begin{equation}\label{eq: gamma w pointwise}
\gamma(W) =  \begin{pmatrix}
\Phi(w^1) \\
\Phi(w^2) \\
\vdots \\
\Phi(w^{c'})
\end{pmatrix} \in \mathbb{R}^{mc' \times mc} \text{, where }
\Phi(w^i) = \begin{pmatrix}
w^i & \cdots & \cdots & \cdots \\
\cdots & w^i & \cdots & \cdots \\
\vdots & \vdots & \vdots & \vdots \\
\cdots & \cdots & \cdots & w^i
\end{pmatrix} \in \mathbb{R}^{m \times mc}.
\end{equation}
Similarly, our goal is to compute the square root of the largest eigenvalue of $\gamma(W)\gamma(W)^\top$. Given~\eqref{eq: gamma w pointwise}, we have
\begin{equation}
     \Phi(w^i)\Phi(w^j)^\top = w^i w^{j \top} I_{m}.
\end{equation}
Hence, we have
\begin{equation}\label{eq: gamma w square pointwise}
\begin{aligned}
    \gamma(W)\gamma(W)^\top &= \begin{pmatrix}
    \Phi(w^1)\Phi(w^1)^\top & \Phi(w^1)\Phi(w^2)^\top & \cdots & \Phi(w^1)\Phi(w^{c'})^\top \\
    \Phi(w^2)\Phi(w^1)^\top & \Phi(w^2)\Phi(w^2)^\top & \cdots & \Phi(w^2)\Phi(w^{c'})^\top \\
    \vdots & \vdots & \vdots & \vdots \\
    \Phi(w^{c'})\Phi(w^1)^\top & \Phi(w^{c'})\Phi(w^2)^\top & \cdots & \Phi(w^{c'})\Phi(w^{c'})^\top 
    \end{pmatrix} \\
    &= \begin{pmatrix}
    w^1w^{1\top} I_{m} & w^1w^{2\top} I_{m} & \cdots & w^1w^{c'\top} I_{m} \\
    w^2w^{1\top} I_{m} & w^2w^{2\top} I_{m} & \cdots & w^2w^{c'\top} I_{m} \\
    \vdots & \vdots & \vdots & \vdots \\
    w^{c'}w^{1\top} I_{m} & w^{c'}w^{2\top} I_{m } & \cdots & w^{c'}w^{c'\top} I_{m} 
    \end{pmatrix} \\
    &\triangleq \Theta(WW^\top, I_{m}),
\end{aligned}
\end{equation}
where we define a new matrix operator $\Theta(\cdot, I_{m})$ in the last step. The following lemma presents a nice property for this operator.

\begin{lemma}\label{lm: theta}
    Given any positive semidefinite matrix $V \in \mathbb{R}^{n \times n}$, let $\lambda = \text{diag}(\lambda_1, ..., \lambda_n)$ denote its eigenvalues. For any $m > 0$, $\Theta(V, I_{m})$ is similar to $\Theta(\lambda, I_{m})$.
\end{lemma}
\begin{proof}
    Since $V$ is positive semidefinite, there exists matrix $P \in \mathbb{R}^{n \times n}$ such that $V$ can be factorized as $P \lambda P^{-1}$, hence $\Theta(V, I_{m}) = \Theta(P \lambda P^{-1}, I_{m})$. We first show that $\Theta$ has distributive property of multiplication, i.e, $\Theta(AB, I_{m}) = \Theta(A, I_{m}) \Theta(B, I_{m})$ for any $A, B \in \mathbb{R}^{n \times n}$. Indeed, we have
    \begin{equation}
        \begin{aligned}
            \Theta(AB, I_{m}) &= \begin{pmatrix}
            (AB)_{1,1} I_{m} & (AB)_{1,2} I_{m} & \cdots & (AB)_{1,n} I_{m} \\
            (AB)_{2,1} I_{m} & (AB)_{2,2} I_{m} & \cdots & (AB)_{2,n} I_{m} \\
            \vdots & \vdots & \vdots & \vdots \\
            (AB)_{n,1} I_{m} & (AB)_{n,2} I_{m} & \cdots & (AB)_{n,n} I_{m} \\
            \end{pmatrix}\\
            \vspace{2mm}
            &= \begin{pmatrix}
            A_{1,1} I_{m} & A_{1,2} I_{m} & \cdots & A_{1,n} I_{m} \\
            A_{2,1} I_{m} & A_{2,2} I_{m} & \cdots & A_{2,n} I_{m} \\
            \vdots & \vdots & \vdots & \vdots \\
            A_{n,1} I_{m} & A_{n,2} I_{m} & \cdots & A_{n,n} I_{m} \\
            \end{pmatrix} \times
            \begin{pmatrix}
            B_{1,1} I_{m} & B_{1,2} I_{m} & \cdots & B_{1,n} I_{m} \\
            B_{2,1} I_{m} & B_{2,2} I_{m} & \cdots & B_{2,n} I_{m} \\
            \vdots & \vdots & \vdots & \vdots \\
            B_{n,1} I_{m} & B_{n,2} I_{m} & \cdots & B_{n,n} I_{m} \\
            \end{pmatrix} \\
            &= \Theta(A, I_{m})\Theta(B, I_{m}).
        \end{aligned}
    \end{equation}
    Given this, we obtain
    \begin{equation}\label{eq:theta 1}
        \Theta(V, I_{m}) = \Theta(P \lambda P^{-1}, I_{m}) = \Theta(P, I_{m})\Theta(\lambda, I_{m})\Theta(P^{-1}, I_{m}).
    \end{equation}
    We also have
    \begin{equation}\label{eq:theta 2}
        \Theta(P, I_{m})\Theta(P^{-1}, I_{m}) = \Theta(PP^{-1}, I_{m}) = I_{mn}.
    \end{equation}
    Combining~\eqref{eq:theta 1} and~\eqref{eq:theta 2} concludes the proof.
\end{proof}
Based on the above lemma, it is straightforward to get $\|\gamma(W)\|_\sigma$ as shown in the following.
\begin{proposition}\label{prop: pointwise} 
Let $W = (w^1, ..., w^{c'}) \in  \mathbb{R}^{c' \times c}$ be a pointwise convolutional weight matrix, then the corresponding fully connected matrix $\gamma(W)$ generated by $W$ satisfies
    \begin{equation}
        \|\gamma(W)\|_{\sigma} = \|W\|_\sigma.
    \end{equation}
\end{proposition}
\begin{proof}
Let $\lambda = \text{diag}(\lambda_1, ..., \lambda_{c'})$, where $\lambda_i$s are the eigenvalues of $WW^\top$. Then by Lemma \ref{lm: theta} and equation~\eqref{eq: gamma w square pointwise}, we have

\begin{equation}
    \begin{aligned}
     \|\gamma(W)\|_\sigma &= \sqrt{\lambda_{\max} \left(\gamma(W)\gamma(W)^\top\right) } = \sqrt{\lambda_{\max} \left(\Theta(WW^\top, I_{m})\right) } \\
     &= \sqrt{\lambda_{\max} \left(\Theta(\lambda, I_{m})\right)} = \sqrt{\max_{i=1,...,c} \lambda_i} = \sqrt{\lambda_{\max} \left(WW^\top\right)} \\
     &= \|W\|_\sigma.
     \end{aligned}
\end{equation}
This completes the proof.

\end{proof}

\section{Comparison with Existing Generalization Bounds}\label{sec-comparison}
In this section, we compare our proposed generalization bounds
with existing bounds, e.g., \citep{neyshabur2015norm, bartlett2017spectrally, neyshabur2017pac, golowich2018size, li2018tighter} both theoretically and empirically. For theoretical comparisons, we discuss generalization bounds for FNNs and FCNNs. For empirical comparisons, we conduct experiments for FCNNs based on MobileNet v1 and MobileNet v2. Note that previous bounds are mainly derived for FNNs without taking into account convolution structures. Thus, for those works which do not explicitly present generalization bounds for FCNNs, we use the transformed fully connected weight matrices generated by convolution weights in place of the fully connected weight matrix in their bounds for FNNs when comparing. 


\subsection{Theoretical Comparisons}
Existing bounds depend on different norms of the fully connected matrices of layers or different norms of the convolutional weights. For instance, the generalization bound of FNNs proposed by \citep{bartlett2017spectrally} depends on the $2,1$-norm and the spectral norm of the fully connected matrix, whereas our bounds depend on its $F$-norm. Besides, with $W$ being the convolutional weight, previous bounds of FCNNs are based on different norms of $\gamma(W)$, whereas ours depend on the $F$-norm of $W$. To make them comparable, we first derive the relation between the relation between $\|A\|_{2,1}$ and $\|A\|_F$ as well as $\|\gamma(W)\|_{2,1}$ and $\|W\|_F$.

\begin{proposition}\label{prop:a 21}
    Given a fully connected weight matrix $A \in \mathbb{R}^{d_{output} \times d_{input}}$ satisfying $\|A\|_F \leq a$, we have
    \begin{equation}
        \|A\|_{2,1} \leq a\sqrt{d_{output}}.
    \end{equation}
\end{proposition}
\begin{proof}
With Jensen's inequality, we have
    \begin{equation}
        \left(\frac{1}{d_{output}} \sum_{i=1}^{d_{output}} \left(\sum_{j=1}^{d_{input}} A_{ij}^2 \right)^{1/2} \right)^2 \leq \frac{1}{d_{output}} \sum_{i=1}^{d_{output}} \left(\sum_{j=1}^{d_{input}} A_{ij}^2 \right).
    \end{equation}
Hence,
    \begin{equation}
        \|A\|_{2,1} = \sum_{i=1}^{d_{output}} \left(\sum_{j=1}^{d_{input}} A_{ij}^2 \right)^{1/2} \leq \sqrt{d_{output}} \|A\|_F \leq a\sqrt{d_{output}}.
    \end{equation}
\end{proof}

\begin{proposition}\label{prop: gamma w f 21 norm}
    Given a convolutional weight matrix $W = (w^1, ..., w^c) \in \mathbb{R}^{c \times r}$ satisfying $\|W\|_F \leq a$. Assuming that each convolutional filter $w^i \in W$ generates $m$ outputs, then its transformed fully connected weight matrix $\gamma(W) \in \mathbb{R}^{d_{output} \times d_{input}}$ where $d_{output} = cm$ satisfying
    \begin{equation}\label{gamma 21 norm}
        \|\gamma(W)\|_{2,1} \leq am\sqrt{c}.
    \end{equation}
\end{proposition}
\begin{proof}
    \begin{equation}\label{eq: 21 norm1}
        \|\gamma(W)\|_{2,1} = \sum_{i=1}^{cm} \left(\sum_{j=1}^{d_{input}} \gamma(W)_{ij}^2 \right)^{1/2} = m \sum_{i=1}^c \left(\sum_{j=1}^r W_{ij}^2\right)^{1/2}.
    \end{equation}
With Jensen's inequality, we have
    \begin{equation}
        \left(\frac{1}{c} \sum_{i=1}^c \left(\sum_{j=1}^r W_{ij}^2\right)^{1/2}\right)^2 \leq \frac{1}{c}  \sum_{i=1}^c  \left(\sum_{j=1}^r W_{ij}^2\right) = \frac{1}{c}\|W\|_F^2.
    \end{equation}
Combining this with~\eqref{eq: 21 norm1}, we get
    \begin{equation}
        \|\gamma(W)\|_{2,1} \leq m\sqrt{c} \|W\|_F \leq am\sqrt{c}.
    \end{equation}
This concludes the proof.
\end{proof}





Based on the summary of existing generalization bounds presented in \citep{li2018tighter}, we list the original and simplified bounds for FNNs in Table \ref{tb: comparison fnn}. 
The simplification is achieved by assuming that, for any layer $i\in\{1,...,L\}$, the Lipschitz constant of activation function $\sigma_i$ is equal to one, the output dimension $d_i$ is equal to $d$, and $\|A_i\|_\sigma \leq s, \|A_i\|_F \leq a$. Then by Proposition \ref{prop:a 21},  we have $\|A_i\|_{2,1} \leq a\sqrt{d}$. Table \ref{tb: comparison fcnn} of generalization bounds for FCNNs is obtained in a similar approach. The simplification column is achieved by further assuming that, for any layer $i\in\{1,...,L\}$, the number of channels denoted by $c_i$ is $c$, the size of each convolutional filter $r_i$ is $r$, and each convolutional filter generates $m$ outputs. We also assume that $\|\gamma_i(W_i)\|_\sigma \leq s$ and $\|W_i\|_F \leq a$. Then by Proposition \ref{prop: gamma w f 21 norm}, we have $\|\gamma_i(W_i)\|_{2,1} \leq a m \sqrt{c}$.


Ignoring constant factors independent of $L$, our result consistently surpasses the results from \citep{bartlett2017spectrally, neyshabur2017pac} in terms of the number of layers $L$ for both cases of FNNs and FCNNs. When $L < \min \{2^L, s^{3L + 1} \}$, which is often the case in FNNs and FCNNs, our result is also tighter than those from \citep{neyshabur2015norm, golowich2018size, li2018tighter}.


\begin{table}[ht]
\scalebox{0.95}{
\begin{tabular}{c|c|c}
 & Original Bound & Simplification \\ \hline
\cite{neyshabur2015norm} & $\mathcal{O}\left(\frac{2^L \prod_{i=1}^L\|A_i\|_F}{\sqrt{n}} \right)$ & $\mathcal{O}\left(\frac{ 2^La^L}{\sqrt{n}} \right)$ \\ \hline
\cite{bartlett2017spectrally} & $\widetilde{\mathcal{O}}\left(\frac{\prod_{i=1}^L \|A_i\|_\sigma}{\sqrt{n}}\left(\sum_{i=1}^L \frac{\|A_i\|_{2,1}^{\frac{2}{3}}}{\|A_i\|_\sigma^{\frac{2}{3}}}\right)^{\frac{3}{2}}\right)$  & $\widetilde{\mathcal{O}}\left(\frac{s^{L-1}L^{\frac{3}{2}}ad^{\frac{1}{2}}}{\sqrt{n}} \right)$ \\ \hline
\cite{neyshabur2017pac} & $\widetilde{\mathcal{O}}\left(\frac{\prod_{i=1}^L \|A_i\|_\sigma}{\sqrt{n}} \sqrt{L^2 d \sum_{i=1}^L \frac{\|A_i\|_F^2}{\|A_i\|_\sigma^2}}\right)$ & $\widetilde{\mathcal{O}}\left(\frac{s^{L-1}L^{\frac{3}{2}}ad^{\frac{1}{2}}}{\sqrt{n}} \right)$ \\ \hline
\cite{golowich2018size} & $\widetilde{\mathcal{O}}\left(\prod_{i=1}^L \|A_i\|_F \cdot  \min\left\{\frac{1}{\sqrt[4]{n}}, \sqrt{\frac{L}{n}}\right\} \right)$ & $\widetilde{\mathcal{O}}\left(a^L \cdot  \min\left\{\frac{1}{\sqrt[4]{n}}, \sqrt{\frac{L}{n}}\right\} \right)$  \\ \hline
\cite{li2018tighter} & $\widetilde{\mathcal{O}}\left(\frac{\prod_{i=1}^L \|A_i\|_\sigma \sqrt{Ld^2}}{\sqrt{n}} \right)$ & $\widetilde{\mathcal{O}}\left(\frac{s^{L}L^{\frac{1}{2}}d}{\sqrt{n}} \right)$  \\ \hline
Our result &  $\widetilde{\mathcal{O}}\left(\frac{\prod_{i=1}^{L} \|A_i\|_\sigma ^{\frac{1}{4}}}{\sqrt{n}}\left(L^2 d^4\sum_{i=1}^{L}\frac{\|A_i\|_F}{\|A_i\|_\sigma}\right)^{\frac{1}{4}}\right)$ & $\widetilde{\mathcal{O}}\left(\frac{s^{\frac{L-1}{4}}L^{\frac{3}{4}}a^{\frac{1}{4}}d}{\sqrt{n}} \right)$ 
\end{tabular}
}\caption{Comparison of generalization bounds for fully connected neural networks.}
\label{tb: comparison fnn}
\end{table}
 
\begin{table}[ht]
\scalebox{0.82}{
\begin{tabular}{c|c|c}
 & Original Bound & Simplification \\ \hline
\cite{neyshabur2015norm} & $\mathcal{O}\left(\frac{2^L \prod_{i=1}^L\|\gamma_i(W_i)\|_F}{\sqrt{n}} \right)$ & $\mathcal{O}\left(\frac{ 2^La^Lm^ \frac{L}{2}}{\sqrt{n}} \right)$ \\ \hline
\cite{bartlett2017spectrally} & $\widetilde{\mathcal{O}}\left(\frac{\prod_{i=1}^L \|\gamma_i(W_i)\|_\sigma}{\sqrt{n}}\left(\sum_{i=1}^L \frac{\|\gamma_i(W_i)\|_{2,1}^{\frac{2}{3}}}{\|\gamma_i(W_i)\|_\sigma^{\frac{2}{3}}}\right)^{\frac{3}{2}}\right)$  & $\widetilde{\mathcal{O}}\left(\frac{s^{L-1}L^{\frac{3}{2}}ac^{\frac{1}{2}}m}{\sqrt{n}} \right)$ \\ \hline
\cite{neyshabur2017pac} & $\widetilde{\mathcal{O}}\left(\frac{\prod_{i=1}^L \|\gamma_i(W_i)\|_\sigma}{\sqrt{n}} \sqrt{L^2 cm \sum_{i=1}^L \frac{\|\gamma_i(W_i)\|_F^2}{\|\gamma_i(W_i)\|_\sigma^2}}\right)$ & $\widetilde{\mathcal{O}}\left(\frac{s^{L-1}L^{\frac{3}{2}}ac^{\frac{1}{2}}m}{\sqrt{n}} \right)$ \\ \hline
\cite{golowich2018size} & $\widetilde{\mathcal{O}}\left(\prod_{i=1}^L \|\gamma_i(W_i)\|_F \cdot  \min\left\{\frac{1}{\sqrt[4]{n}}, \sqrt{\frac{L}{n}}\right\} \right)$ & $\widetilde{\mathcal{O}}\left(a^Lm^{\frac{L}{2}} \cdot  \min\left\{\frac{1}{\sqrt[4]{n}}, \sqrt{\frac{L}{n}}\right\} \right)$  \\ \hline
\cite{li2018tighter} & $\widetilde{\mathcal{O}}\left(\frac{\prod_{i=1}^L \|\gamma_i(W_i)\|_\sigma \sqrt{Lc^2m^2}}{\sqrt{n}} \right)$ & $\widetilde{\mathcal{O}}\left(\frac{s^{L}L^{\frac{1}{2}}cm}{\sqrt{n}} \right)$  \\ \hline
Our result &  $\widetilde{\mathcal{O}}\left(\frac{\prod_{i=1}^{L} \|\gamma_i(W_i)\|_\sigma ^{\frac{1}{4}}}{\sqrt{n}}\left(L^2 c^2r^2\sqrt{m}\sum_{i=1}^{L}\frac{\|W_i\|_F}{\|\gamma_i(W_i)\|_\sigma}\right)^{\frac{1}{4}}\right)$ & $\widetilde{\mathcal{O}}\left(\frac{s^{\frac{L-1}{4}}L^{\frac{3}{4}}a^{\frac{1}{4}}c^{\frac{1}{2}}m^{\frac{1}{8}}r^{\frac{1}{2}}}{\sqrt{n}} \right)$ 
\end{tabular}
}\caption{Comparison of generalization bounds for fully convolutional neural networks.}
\label{tb: comparison fcnn}
\end{table}

\subsection{Empirical Comparisons}
In this section, we conduct experiments to empirically demonstrate the advantage of our generalization bounds. We use MobileNet V1 \citep{howard2017mobilenets} and V2 \citep{sandler2018mobilenetv2} to compare our FCNN bounds with others. MobileNets extensively utilize depthwise separable convolutions to achieve balance between efficiency and accuracy such that they are suitable to deploy to mobile devices with limited computation power. In this experiment, we use official trained weights of MobileNets \footnote{https://github.com/tensorflow/models/tree/master/research/slim} to instantiate the network. For simplicity, we bound $\|\gamma(W)\|_\sigma$, $\|\gamma(W)\|_F$, and $\|\gamma(W)\|_{2,1}$ in terms of $\|W\|_F$ according to Proposition \ref{prop: gamma w sigma standard}, \ref{prop: non depthwise}, \ref{prop: over depthwise}, \ref{prop: pointwise} and \ref{prop: gamma w f 21 norm}. 
Additionally, we ignore the factors of the number of training sample $n$, as they are the same for all generalization bounds discussed here. Based on the results shown in Figure \ref{fig:comparison mobilenets}, we can see that our bound is much tighter than others empirically.

\begin{figure}

\begin{subfigure}[t]{0.5\textwidth}
\begin{tikzpicture}[font=\tiny]
    \begin{semilogyaxis}[
        symbolic x coords={1, 2, 3, 4, 5, 6},
        x tick label style={ /pgf/number format/1000 sep=},
        nodes near coords,
        every node near coord/.append style={font=\tiny, color=black},
        ]
    \addplot+[ybar, fill=blue, mark=none, point meta=explicit symbolic] table[meta=label] { 
        x y label
        1 1.93e83 1.93e83
        2 4.32e30  4.32e30
        3 8e30 8e30
        4 3.81e75 3.81e75
        5 3.41e30 3.41e30
        6 1.6e10 1.6e10
    };
    \end{semilogyaxis}
\end{tikzpicture}
\caption{MobileNet v1}
\label{fig:comparison mnv1}
\end{subfigure}
~
\begin{subfigure}[t]{0.5\textwidth}
\begin{tikzpicture}[font=\tiny]
    \begin{semilogyaxis}[
        symbolic x coords={1, 2, 3, 4, 5, 6},
        x tick label style={ /pgf/number format/1000 sep=},
        nodes near coords,
        every node near coord/.append style={font=\tiny, color=black},
        ]
    \addplot+[ybar, fill=blue, mark=none, point meta=explicit symbolic] table[meta=label] { 
        x y label
        1 2.41e144 2.41e144
        2 7.52e40 7.52e40 
        3 1.88e41 1.88e41
        4 1.94e129 1.94e129
        5 6.85e40 6.85e40
        6 5.1e12 5.1e12
    };
    \end{semilogyaxis}
\end{tikzpicture}
\caption{MobileNet v2}
\label{fig:comparison mnv2}
\end{subfigure}
\caption{Comparison of generalization bounds for MobileNets. From left to right, the bounds represent \cite{neyshabur2015norm},  \cite{bartlett2017spectrally}, \cite{neyshabur2017pac}, \cite{golowich2018size}, \cite{li2018tighter} and our result, respectively.}
\label{fig:comparison mobilenets}

\end{figure}

\section{Conclusion}\label{sec-conclude}
In this paper, we propose a margin-based generalization bound for general convolutional neural networks that can have both fully connected layers and convolutional layers. We study spectral norm for fully connected matrices generated by three types of convolution operations including standard convolution, depthwise convolution, and pointwise convolution. We show that the proposed generalization bounds for MobileNets are indeed tighter compared with existing bounds from both theoretical and experimental views. Such advantage is achieved by exploring the sparsity and shared weights of convolutional layers. 

In the end, it is worth mentioning that generalization bounds derived for convolutional neural networks cannot fully explain generalization phenomenon. More effort is needed to improve our theoretical understanding on generalization. We believe that generalization is achieved via combined efforts from optimization algorithms, loss functions, the structure of networks, and other factors. It is of interest to study how these factors contribute to generalization individually and jointly.

\vskip 0.2in
\bibliography{sample}

\end{document}